\documentclass{article}

\usepackage{microtype}
\usepackage{graphicx}
\usepackage{subfigure}
\usepackage{booktabs} % for professional tables

\usepackage{amsfonts}       % blackboard math symbols
\usepackage{amsmath}
\usepackage{amsthm}
\usepackage{ amssymb }
\usepackage{todonotes}
\usepackage{lipsum}
\usepackage{bbold}
\usepackage{makecell}
\usepackage{color, colortbl}

\usepackage{enumitem}

\usepackage{hyperref}

\newtheorem{theorem}{Theorem}
\theoremstyle{definition}
\newtheorem{definition}{Definition}

\newcommand{\macrocite}[1]{\cite{#1}}
\newcommand{\suppcite}[1]{\cite{#1}}
\newcommand{\mainrefs}{\bibliographystyle{icml2021}\bibliography{references}}
\newcommand{\supprefs}{}

\usepackage[preprint]{icml2021}

\icmltitlerunning{\textsc{FedAUX}: Leveraging Unlabeled Auxiliary Data in Federated Learning}

\begin{document}

\twocolumn[
\icmltitle{\textsc{FedAUX}: Leveraging Unlabeled Auxiliary Data in Federated Learning}
%\icmltitle{\textsc{FedAUX}: Deriving utility from unlabeled Auxiliary data in Federated Distillation}

% It is OKAY to include author information, even for blind
% submissions: the style file will automatically remove it for you
% unless you've provided the [accepted] option to the icml2021
% package.

% List of affiliations: The first argument should be a (short)
% identifier you will use later to specify author affiliations
% Academic affiliations should list Department, University, City, Region, Country
% Industry affiliations should list Company, City, Region, Country

% You can specify symbols, otherwise they are numbered in order.
% Ideally, you should not use this facility. Affiliations will be numbered
% in order of appearance and this is the preferred way.
\icmlsetsymbol{equal}{*}

\begin{icmlauthorlist}
\icmlauthor{Felix Sattler}{to}
\icmlauthor{Tim Korjakow}{to}
\icmlauthor{Roman Rischke}{to}
\icmlauthor{Wojciech Samek}{to}

\end{icmlauthorlist}

\icmlaffiliation{to}{Department of Artificial Intelligence, Fraunhofer HHI, Berlin, Germany}

\icmlcorrespondingauthor{Felix Sattler}{\mbox{felix.sattler@hhi.fraunhofer.de}}
\icmlcorrespondingauthor{Wojciech Samek}{\mbox{wojciech.samek@hhi.fraunhofer.de}}

% You may provide any keywords that you
% find helpful for describing your paper; these are used to populate
% the "keywords" metadata in the PDF but will not be shown in the document
\icmlkeywords{Machine Learning, ICML}

\vskip 0.3in
]

% this must go after the closing bracket ] following \twocolumn[ ...

% This command actually creates the footnote in the first column
% listing the affiliations and the copyright notice.
% The command takes one argument, which is text to display at the start of the footnote.
% The \icmlEqualContribution command is standard text for equal contribution.
% Remove it (just {}) if you do not need this facility.

%\printAffiliationsAndNotice{}  % leave blank if no need to mention equal contribution
\printAffiliationsAndNotice{} % otherwise use the standard text.

\begin{abstract}
Federated Distillation (FD) is a popular novel algorithmic paradigm for Federated Learning, which achieves training performance competitive to prior parameter averaging based methods, while additionally allowing the clients to train different model architectures, by distilling the client predictions on an \emph{unlabeled auxiliary set of data} into a student model. 
In this work we propose \textsc{FedAUX}, an extension to FD, which, under the same set of assumptions, drastically improves performance by deriving maximum utility from the unlabeled auxiliary data. 
\textsc{FedAUX} modifies the FD training procedure in two ways: First, unsupervised pre-training on the auxiliary data is performed to find a model initialization for the distributed training. Second, $(\varepsilon, \delta)$-differentially private certainty scoring is used to weight the ensemble predictions on the auxiliary data according to the certainty of each client model. 
Experiments on large-scale convolutional neural networks and transformer models demonstrate, that the training performance of \textsc{FedAUX} exceeds SOTA FL baseline methods by a substantial margin in both the iid and non-iid regime, further closing the gap to centralized training performance. Code is available at \url{github.com/fedl-repo/fedaux}.
\end{abstract}

\section{Introduction}
\label{sec:intro}
\begin{figure*}[t!]
    \centering
    \includegraphics[width=\textwidth]{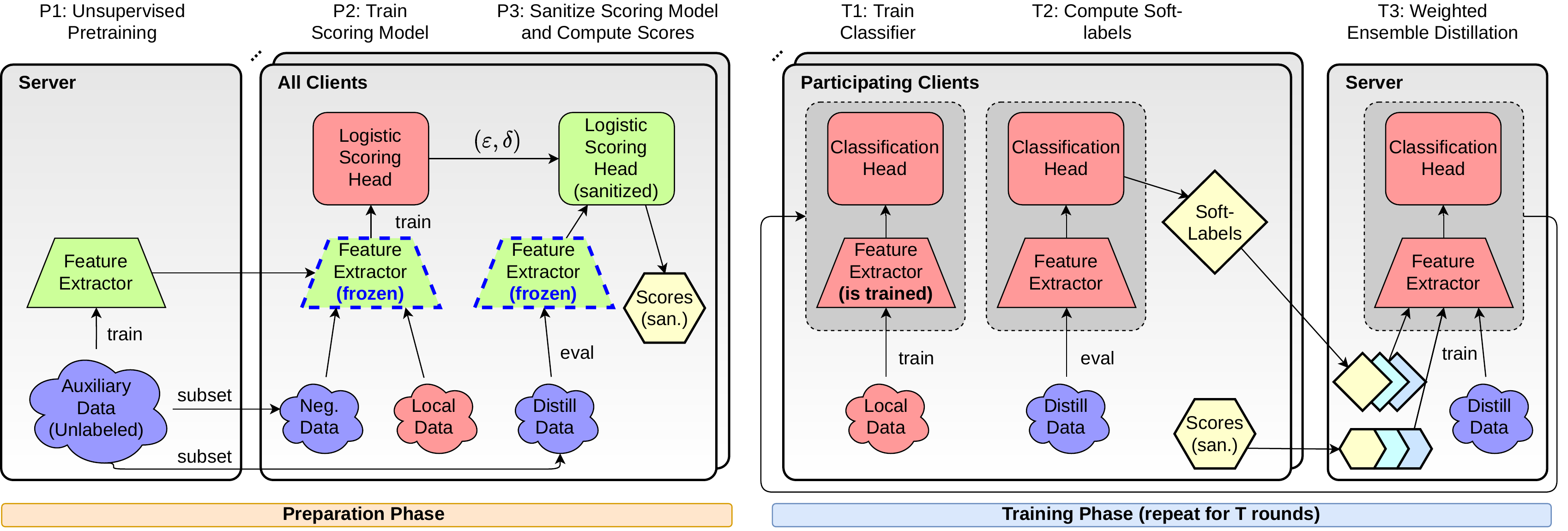}
    \vspace{-0.8cm}
    \caption{Training procedure of \textsc{FedAUX}. \textbf{Preparation phase:} P1) The unlabeled auxiliary data is used to pre-train a feature extractor (e.g. using contrastive representation learning). P2) The feature-extractor is sent to the clients, where it is used to initialize the client models. Based on extracted features, a logistic scoring head is trained to distinguish local client data from a subset of the auxiliary data. P3) The trained scoring head is sanitized using a $(\varepsilon, \delta)$-differentially private mechanism and then used to compute certainty scores on the distillation data.
    \textbf{Training Phase:} T1) In each communication round, a subset of the client population is selected for training. Each selected client downloads a model initialization from the server, and then updates the full model $f_i$ (feature extractor \& scoring head) using their private local data. T2) The locally trained classifier and scoring models $f_i$ and $s_i$ are sent to the server, where they are combined into a weighted ensemble. T3) Using the unlabeled auxiliary data and the weighted ensemble as a teacher, the server distills a student model which is used as the initialization point for the next round of Federated training. \textbf{*}Note that in practice we perform computation of soft-labels and scores at the server to save client resources.}
    \label{fig:overview}
\end{figure*}

Federated Learning (FL) allows distributed entities ("clients") to jointly train (deep) machine learning models on their combined data, without having to transfer this data to a centralized location \macrocite{mcmahan2017communication}. The Federated training process is orchestrated by a central server. The distributed nature of FL improves privacy \macrocite{li2019federated}, ownership rights \macrocite{sheller2020federated} and security \macrocite{mothukuri2020survey} for the participants. As the number of mobile and IoT devices and their capacities to collect large amounts of high-quality and privacy-sensitive data steadily grows, Federated training procedures become increasingly relevant. 

While the client data in Federated Learning is typically assumed to be private, in most real-world applications the server additionally has access to unlabeled \emph{auxiliary} data, which roughly matches the distribution of the client data. For instance, for many Federated computer vision and natural language processing problems, such auxiliary data can be given in the form of public data bases such as ImageNet \macrocite{deng2009imagenet} or WikiText \macrocite{merity2016pointer}. These data bases contain millions to billions of data samples but are typically lacking the necessary label information to be useful for training task-specific models.

Recently, Federated Distillation (FD), a novel algorithmic paradigm for Federated Learning problems where such auxiliary data is available, was proposed. In contrast to classic parameter averaging based FL algorithms \macrocite{mcmahan2017communication, mohri2019agnostic, reddi2020adaptive, li2019fedprox, sattler2019robust}, which require all client's models to have the same size and structure, FD allows the clients to train heterogeneous model architectures, by distilling the client predictions on the auxiliary set of data into a student model. This can be particularly beneficial in situations where clients are running on heterogeneous hardware. Studies show that FD based training has favorable communication properties \macrocite{itahara2020distill, sattler2020communication}, and can outperform parameter averaging based algorithms \macrocite{lin2020ensemble}. 

However, just like for their parameter-averaging-based counterparts, the performance of FD based learning algorithms falls short of centralized training and deteriorates quickly if the training data is distributed in a heterogeneous ("non-iid") way among the clients. In this work we aim to further close this performance gap, by exploring the core assumption of FD based training and deriving maximum utility from the available unlabeled auxiliary data. Our main contributions are as follows:
\begin{itemize}
    \item We show that a wide range of (out-of-distribution) auxiliary data sets are suitable for self-supervised pre-training and can drastically improve FL performance across all baselines. 
    \item We propose a novel certainty-weighted FD technique, that improves performance of FD on non-iid data substantially, addressing a long-standing problem in FL research.
    \item We propose an $(\varepsilon, \delta)$-differentially private mechanism to constrain the privacy loss associated with transmitting certainty scores. 
\end{itemize}
These performance improvements are possible a) under the same assumptions made in the FD literature, b) with only negligible additional computational overhead for the resource-constrained clients and c) with small quantifiable excess privacy loss.

\section{Related Work}
\label{sec:related}
%A variety of studies have exploited unlabeled auxiliary data in the context of FL.
\textbf{Federated Distillation:}
Distillation \macrocite{bucila2006compression, hinton2015distill} is a common technique to transfer the knowledge of one or multiple \macrocite{you2017learning, anil2018large} machine learning classifiers to a different model, and is typically used in centralized settings before deployment in order to reduce the model complexity, while preserving predictive power. To this end, the predictions of the teacher model(s) on a distillation data set are used to guide the training process of the potentially less complex student model. Federated Distillation (FD) algorithms, which leverage these distillation techniques to aggregate the client knowledge, are recently gaining popularity, because they outperform conventional parameter averaging based FL methods \macrocite{lin2020ensemble, chen2020feddistill} like \textsc{FedAVG} or FedPROX \macrocite{mcmahan2017communication, li2019fedprox} and allow clients to train heterogeneous model architectures \macrocite{li2019fedmd, chang2019cronus, li2021fedh2l}. FD methods can furthermore reduce communication overhead \macrocite{jeong2018distill, itahara2020distill, seo2020fd, sattler2020communication}, by exploiting the fact that distillation requires only the communication of model predictions instead of full models. In contrast to centralized distillation, where training and distillation data usually coincide, FD makes no restrictions on the auxiliary distillation data\footnote{Recent work even suggests that useful distillation data can be generated from the teacher models themselves \macrocite{nayak2019zero}.}, 
making it widely applicable.
% Our work aims to improve overall training performance in FL and is most closely related to the \textsc{FedDF} training protocol proposed in \macrocite{lin2020ensemble}, where distillation is combined with parameter averaging to achieve SOTA training performance. Building upon their work, we additionally leverage the distillation data set for unsupervised pre-training and weight the client predictions in the distillation step according to their prediction certainty to better cope with settings where the client’s data generating distributions are statistically heterogeneous. 
% Our work also bears similarity with FedBE \macrocite{chen2020feddistill}, which combines client predictions by means of a Bayesian model ensemble. While their proposed aggregation method achieves performance improvements over \textsc{FedAVG} in non-iid data settings, it does not support heterogeneous client models and cannot incorporate client prediction certainty. 
Our work, is in line with \macrocite{lin2020ensemble, chen2020feddistill} in that it aims to improve overall training performance in FL. Both \textsc{FedDF} \macrocite{lin2020ensemble} and \textsc{FedBE} \macrocite{chen2020feddistill} combine parameter averaging as done in FedAVG \macrocite{mcmahan2017communication} with ensemble distillation to improve FL performance. While \textsc{FedDF} combines client predictions by means of an (equally weighted) model ensemble, \textsc{FedBE} forms a Bayesian ensemble from the client models for better robustness to heterogeneous data. Taking \textsc{FedDF} as a starting point, we additionally leverage the auxiliary distillation data set for unsupervised pre-training and weight the client predictions in the distillation step according to their prediction certainty to better cope with settings where the client’s data generating distributions are statistically heterogeneous. 

% One line of research  exploits the fact that FD requires only the communication of model predictions instead of full models, to drastically reduce communication overhead. In doing so FD techniques also allow the clients to train different model architectures. See \macrocite{sattler2020communication} for a detailed comparison of FedAvg and FD and an analysis of the communication properties of the latter algorithmic framework. We also mention the related work by Guha~et~al.~\macrocite{guha2019one-shot}, which proposes a one-shot distillation method for convex models, where the server distills the locally optimized client models in a single round.

% \textbf{Unlabeled Data in FL:}
% Federated semi-supervised learning techniques \macrocite{zhang2020benchmarking, jeong2020federated} assume that clients hold both labeled and unlabeled private data from the same distribution. In contrast, we assume that the server has access to public unlabeled data that differs in distribution from the local client data. Federated self-supervised representation learning aims to train a feature extractor on private unlabeled client data \macrocite{zhang2020federated}. In contrast, we leverage self-supervised representation learning at the server to find a suitable model initialization.

\textbf{Weighted Ensembles:} 
Weighted ensemble methods were studied already in classical work \macrocite{hashem1993ensemble, perrone1993ensemble, opitz1999popular}, with certainty weighted ensembles of neural networks in particular being proposed for classification e.g. in \macrocite{jimenez1998dynamically}.
Mixture of experts and boosting methods \macrocite{yuksel2012twenty, masoudnia2014mixture, schapire1999brief} where multiple simple classifiers are combined by weighted averaging are frequently used in centralized settings.
% Domain adaptation theory \macrocite{mansour2008domain, hoffman2018domain} suggest that a distribution-weighted combination of the local hypotheses (teacher models), yields more robust domain adaptation. We address exactly this open question for federated ensemble distillation via privacy-preserving local distribution estimation. \todo{this also appears in sec. \ref{sec:weighted_distill}, remove here?}

A more detailed discussion of related work can be found in Appendix \ref{supp:related}.

\section{Federated Learning with Auxiliary Data}
\label{sec:method}

\begin{figure*}[t!]
    \centering
    \includegraphics[width=1.0\textwidth]{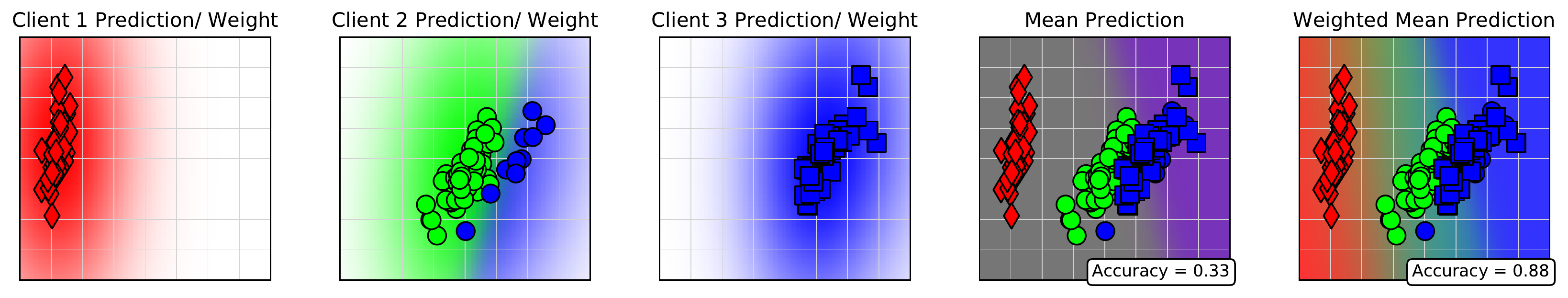}
    \vspace{-0.8cm}
    \caption{\textbf{Weighted Ensemble Distillation} illustrated in a toy example on the Iris data set (data points are projected to their two principal components). Three Federated Learning clients hold disjoint non-iid subsets of the training data. Panels 1-3: Predictions made by linear classifiers trained on the data of each client. Labels and predictions are color-coded, client certainty (measured via Gaussian KDE) is visualized via the alpha-channel. The mean of client predictions (panel 4) only poorly captures the distribution of training data. In contrast, the certainty-weighted mean of client predictions (panel 5) achieves much higher accuracy.}
    \label{fig:toy}
\end{figure*}

%\begin{algorithm}[t!]
%\caption{}\label{alg:FED}
%\For{$t=1,..,T$}{
%\For{$i \in \mathcal{S}_t\subseteq \{1,..,N\}$ \textbf{in parallel}}{
%\underline{Client $C_i$ does:}\\
%\textbullet~ $\theta_i\leftarrow \text{train}(\theta_0\leftarrow\theta, X_i, Y_i)$\hfill \# Local Training\\
%\textbullet~ $Y^{aux}_i\leftarrow f_{\theta_i}(X^{aux})$ \hfill\# Prediction\\
%\textbullet~ $s^{aux}_i\leftarrow f^s_{\theta_i}(X^{aux})$ \hfill \# Compute scores
%}
%\underline{Server $S$ does:}\\
%\textbullet~$\theta\leftarrow \sum_{i\in \mathcal{S}_t}\frac{|D_i|}{\sum_{l\in \mathcal{S}_t}|D_l|} \theta_i$\\
%\textbullet~$Y^{aux}\leftarrow\mathcal{A}_{mean}(Y^{aux}_i|i\in \mathcal{S}_t)$ \hfill \# Softlabel Aggregation\\
%\textbullet~$Y^{aux}\leftarrow\mathcal{A}_{weighted}(Y^{aux}_i, s^{aux}_i|i\in \mathcal{S}_t)$ \hfill \# Weighted Softlabel Aggregation\\
%\textbullet~ $\theta\leftarrow \text{train}(\theta_0\leftarrow\theta, X^{aux}, Y^{aux})$\hfill \# Distillation\\
%}
%\end{algorithm}

In this section, we describe our method for efficient Federated Learning in the presence of unlabeled auxiliary data (\textsc{FedAUX}). An illustration of our proposed approach is given in Figure \ref{fig:overview}. We describe \textsc{FedAUX} for the homogeneous setting were all clients hold the same model prototype. The detailed algorithm for the more general model-heterogeneous setting can be found in Appendix \ref{supp:algorithm}. An exhaustive \emph{qualitative} comparison between \textsc{FedAUX} and baseline methods is given in Appendix \ref{supp:qualitative}. 

\subsection{Problem Setting} We assume the conventional FL setting where a population of $n$ clients is holding potentially non-iid subsets of private labeled data $D_1,..,D_n$, from a training data distribution $(\bigcup_{i\leq n}D_i)\sim \varphi(\mathcal{X}, \mathcal{Y})$. We further make the assumption that the server and the clients both have access to a public collection of unlabeled auxiliary data from a deviating distribution $D_{aux}\sim \psi(\mathcal{X})$. The latter assumption is common to all studies on FD. 

One round of federated training is then performed as follows: A subset $\mathcal{S}_t$ of the client population is selected by the server and downloads a model initialization. Starting from this model initialization, each client then proceeds to train a model $f_i$ on it's local private data $D_i$ by taking multiple steps of stochastic gradient descent. We assume that these local models can be decomposed into a feature extractor $h_i$ and a classification head $g_i$ according to $f_i=g_i \circ h_i$%\footnote{For better readability, in this study we synonymize models with their parameters.}
. Finally, the updated models $f_i$, $i\in\mathcal{S}_t$ are sent back to the server, where they are aggregated to form a new server model $f$, which is used as the initialization point for the next round of FL. The goal of FL is to obtain a server model $f$, which optimally generalizes to new samples from the training data distribution $\varphi$, within a minimum number of communication rounds $t\leq T$.
%Every client is training a model $f_i : \mathcal{X}, \Theta \rightarrow \mathcal{Y}$, which .

\subsection{Federated Ensemble Distillation}
Federated Ensemble Distillation is a novel method for aggregating the knowledge of FL clients. Instead of aggregating the parameters of the client models (e.g. via an averaging operation), a student model is trained on the combined predictions of the clients on some public auxiliary data. Let $x\in D_{aux}$ be a batch of data from the auxiliary distillation data set. Then one iteration of student distillation is performed as
\begin{align}
\label{eq:distill_update}
    \theta^{t, j+1} \leftarrow \theta^{t, j} - \eta\frac{\partial D_{KL}(\mathcal{A}(\{f_i(x)|i\in \mathcal{S}_t\}), \sigma(f(x, \theta^{t, j})))}{\partial \theta^{t, j}}.
\end{align}
Hereby, $D_{KL}$ denotes the Kullback-Leibler divergence, $\eta>0$ is the learning rate, $\sigma$ is the softmax-function and $\mathcal{A}$ is a mechanism to aggregate the soft-labels. Existing work \macrocite{lin2020ensemble} aggregates the client predictions by taking the mean according to
\begin{align}
\label{eq:agg_mean}
\mathcal{A}_{mean}(\{f_i(x)|i\in \mathcal{S}_t\}) =  \sigma\left(\frac{\sum_{i\in \mathcal{S}_t} f_i(x)}{|\mathcal{S}_t|}\right).
\end{align}
Federated Ensemble Distillation is shown to outperform parameter averaging based techniques \macrocite{lin2020ensemble}.

\subsection{Self-supervised Pre-training}
\label{sec:pretrain}

Self-supervised representation learning can leverage large records of unlabeled data to create models which extract meaningful features. For the two types of data considered in this study - image and sequence data - strong self-supervised training algorithms are known in the form of contrastive representation learning \macrocite{chen2020simple, wang2020understanding} and next-token prediction \macrocite{devlin2018bert, radford2019language}. 
As part of the \textsc{FedAUX} preparation phase (cf. Fig. \ref{fig:overview}, P1) we propose to perform self-supervised training on the auxiliary data $D_{aux}$ at the server. We emphasize that this step makes no assumptions on the similarity between the local training data and the auxiliary data. This results in a parametrization for the feature extractor $h_0$. Since the training is performed at the server, using publicly available data, this step inflicts neither computational overhead nor privacy loss on the resource-constrained clients.

\subsection{Weighted Ensemble Distillation}
\label{sec:weighted_distill}
Different studies have shown that both the training speed, stability and maximum achievable accuracy in existing FL algorithms deteriorate if the training data is distributed in a heterogeneous "non-iid" way among the clients \macrocite{zhao2018federated, sattler2019robust, li2020covergence}. Federated Ensemble Distillation makes no exception to this rule \macrocite{lin2020ensemble}.

The underlying problem of combining hypotheses derived from different source domains has been explored in multiple-source domain adaptation theory \macrocite{mansour2008domain, hoffman2018domain}, which shows that standard convex combinations of the hypotheses of the clients as done in \macrocite{lin2020ensemble} may perform poorly on the target domain. Instead, a distribution-weighted combination of the local hypotheses is shown to be robust \macrocite{mansour2008domain, hoffman2018domain}. A simple toy example, displayed in Figure~\ref{fig:toy}, further illustrates this point.

%For a better intuition of this problem, let us consider the toy example illustrated in Figure \ref{fig:toy}. Displayed as scatter points are elements of the Iris data set, projected to their two main PCA components. The training data is distributed among three clients in a non-iid fashion, the label of each data point is indicated by the marker color. Overlayed in the background are the predictions of linear classifier  models that were trained on the local data of each client. As we can see, the models which were trained on the data of clients 1 and 3, uniformly predict that all inputs belong to the "red" and "blue" class respectively. The predictive power of these models and consequently their value as teachers for model distillation is thus very limited. This is also visualized in panel 4, where the mean prediction of the teacher models is displayed. We can however improve the teacher ensemble quite significantly, if we weight each teachers predictions at every location $x$ by it's certainty $s(x)$, illustrated via the alpha channel in panels 1-3. In the given example we approximate the value of $s(x)$ via Gaussian KDE. 

Inspired by these results, we propose to modify the aggregation rule of FD \eqref{eq:agg_mean} to a certainty-weighted average:
\begin{align}
\label{eq:weighted}
\mathcal{A}_{s}(\{(f_i(x), s_i(x))|i\in \mathcal{S}_t\}) = \sigma\left(\frac{\sum_{i\in \mathcal{S}_t} s_i(x)f_i(x)}{\sum_{i\in \mathcal{S}_t} s_i(x)}\right)
\end{align}
The question remains, how to calculate the certainty scores $s_i(x)$ in a privacy preserving way and for arbitrary high-dimensional data, where simple methods, such as Gaussian KDE used in our toy example, fall victim to the curse of dimensionality.  To this end, we propose the following methodology:

We split the available auxiliary data randomly into two disjoint subsets, 
$
    D^-~\cup~ D_{distill}= D_{aux},
$
the "negative" data and the "distillation" data.  Using the pre-trained model $h_0$ ($\rightarrow$ sec. \ref{sec:pretrain}) as a feature extractor, on each client, we then train a logistic regression classifier to separate the local data $D_i$ from the negatives $D^-$, by optimizing the following regularized empirical risk minimization problem
\begin{align}
\label{eq:ERM}
    w_i^* = \arg\min_{w} J(w, h_0, D_i, D^-)
\end{align}
with 
\begin{align}
\begin{split}
    J(w, h_0, D_i, D^-) = &a\sum_{x\in D_i \cup D^-}l(t_x\langle w, \tilde{h}_0(x)\rangle)+\lambda R(w).
\end{split}
\end{align}
Hereby $t_x=2(\mathbb{1}_{x\in D_i})-1\in[-1,1]$ defines the binary labels of the separation task, $a=(|D_i|+|D^-|)^{-1}$ is a normalizing factor and $\tilde{h}_0(x)=h_0(x)(\max_{x\in D_i\cup D^-}\|h_0(x)\|)^{-1}$ are the normalized features. We choose  $l(z) = \log(1+\exp(z))$ to be the logistic loss and $R(w) = \frac{1}{2}\|w\|^2_2$ to be the $\ell_2$-regularizer. Since $J$ is $\lambda$-strongly convex in $w$, problem \eqref{eq:ERM} is uniquely solvable. This step is performed only once on every client, during the preparation phase (cf. Fig.~\ref{fig:overview}, P2) and the computational overhead for the clients of solving \eqref{eq:ERM} is negligible in comparison to the cost of multiple rounds of training the (deep) model $f_i$.

Given the solution of the regularized ERM $w_i^*$, the certainty scores on the distillation data $D_{distill}$ can be obtained via 
\begin{align}
s_i(x)=(1+\exp(-\langle w_i^*, \tilde{h}_0(x)\rangle))^{-1}+\xi.
\end{align}
A small additive $\xi>0$ ensures numerical stability when taking the weighted mean in \eqref{eq:weighted} (we set $\xi=1e-8$). In Appendix \ref{supp:domain_adaptation}, we provide further empirical results, suggesting that our certainty-weighted averaging method \eqref{eq:weighted} approximates a robust aggregation rule proposed in \macrocite{mansour2008domain}.%\todo{change formulation? R: I like it.}

% \todo{ Link to Appendix, where we analyze how well the certainty scores approximate simple distribution-weights suggested by multiple-source domain adaptation theory \macrocite{mansour2008domain, hoffman2018domain}. That is, we show for an illustrating example of a Gaussian mixture that $s_i(x) / \sum_{j} s_j(x) \approx D_i(x) / \sum_j D_j(x)$ for $x \in \mathcal{X}$, where $D_i$ is the local data generating distribution of client $i$. }

\subsection{Privacy Analysis}
\label{sec:privacy}
Sharing the certainty scores $\{s_i(x)|x\in D_{distill}\}$ with the central server intuitively causes privacy loss for the clients. After all, a high score $s_i(x)$ indicates, that the public data point $x\in D_{distill}$ is similar to the private data $D_i$ of client $i$ (in the sense of \eqref{eq:ERM}). To protect the privacy of the clients, quantify and limit the privacy loss, we propose to use data-level differential privacy (cf. Fig. \ref{fig:overview}, P3). Following the classic definition of \macrocite{dwork2014algorithmic}, a randomized mechanism is called differentially private, if it's output on any input data base $d$ is indistinguishable from output on  any neighboring database $d'$ which differs from $d$ in one element.

\begin{definition}
A randomized mechanism $\mathcal{M} : \mathcal{D} \rightarrow{\mathcal{R}}$ satisfies $(\varepsilon,\delta)$-differential privacy if for any two adjacent inputs $d$ and $d'$ that differ in only one element and for any subset of outputs $S\subseteq \mathcal{R}$, it holds that
\begin{align}
    P[\mathcal{M}(d)\in S]\leq \exp(\varepsilon)P[\mathcal{M}(d')\in S]+\delta.
\end{align}
\end{definition}

Differential privacy of a mechanism $\mathcal{M}$ can be achieved, by limiting it's sensitivity 
\begin{align}
\Delta(\mathcal{M}) = \max_{d_1,d_2\in\mathcal{D}}\|\mathcal{M}(d_1)-\mathcal{M}(d_2)\|
\end{align}
 and then applying a randomized noise mechanism. We adapt a Theorem from \macrocite{chaudhuri2011differentially} to establish the sensitivity of \eqref{eq:ERM}: 

\begin{theorem}
\label{theo:1}
If $R(\cdot)$ is differentiable and 1-strongly convex and $l$ is differentiable with $|l'(z)|\leq 1$ $\forall z$, then the $\ell^2$-sensitivity $\Delta_2(\mathcal{M})$ of the mechanism 
\begin{align}
    \mathcal{M} : D_i \mapsto \arg\min_{w} J(f, h_0, D_i, D^-)
\end{align}
is at most $2(\lambda(|D_i|+|D^-|))^{-1}$.
\end{theorem}

The proof can be found in Appendix \ref{supp:proof}. As we can see the sensitivity scales inversely with the size of the total data $|D_i|+|D^-|$. From Theorem \ref{theo:1} and application of the Gaussian mechanism \macrocite{dwork2014algorithmic} it follows that the randomized mechanism
\begin{align}
\label{eq:san}
    \mathcal{M}_{san} : D_i \mapsto \arg\min_{f} J(f, h_0, D_i, D^-)+N
\end{align}
with $N\sim\mathcal{N}(\mathbf{0}, I\sigma^2)$  and  $\sigma^2=\frac{8\ln(1.25\delta^{-1})}{\varepsilon^2\lambda^2(|D_i|+|D_{aux}|)^2}$ is $(\varepsilon, \delta)$-differentially private.

The post-processing property of DP ensures that the release of any number of scores computed using the output of mechanism $\mathcal{M}_{san}$ is still $(\varepsilon,\delta)$-private. Note, that in this work we restrict ourselves to the privacy analysis of the scoring mechanism. The differentially private training of deep classifiers $f_i$ is a challenge in it's own right and has been addressed e.g. in \macrocite{abadi2016deep}. Following the basic composition theorem \macrocite{dwork2014algorithmic}, the total privacy cost of running \textsc{FedAUX} is the sum of the privacy loss of the scoring mechanism $\mathcal{M}_{san}$ and the privacy loss of communicating the updated models $f_i$ (the latter is the same for all FL algorithms).

\section{Experiments}
\subsection{Setup}
\textbf{Datasets and Models:} We evaluate \textsc{FedAUX} and SOTA FL methods on both Federated image and text classification problems with large scale convolutional and transformer models respectively. For our image classification problems we train ResNet- \macrocite{he2016deep}, MobileNet- \macrocite{sandler2018mobilenetv2} and ShuffleNet- \macrocite{zhang2018shufflenet} type models on CIFAR-10 and CIFAR-100 and use STL-10, CIFAR-100 and SVHN as well as different subsets of ImageNet (Mammals, Birds, Dogs, Devices, Invertebrates, Structures)\footnote{The methodology for generating these subsets is described in Appendix \ref{supp:iamgenet_subsets}} as auxiliary data. In our experiments, we always use 80\% of the auxiliary data as distillation data $D_{distill}$ and 20\% as negative data $D^-$. For our text classification problems we train Tiny-Bert \macrocite{jiao2020tinybert} on the AG-NEWS \macrocite{Zhang2015CharacterlevelCN} and Multilingual Amazon Reviews Corpus \macrocite{marc_reviews} and use BookCorpus \macrocite{Zhu_2015_ICCV} as auxiliary data.

\begin{figure*}[t!]
    \centering
    \includegraphics[width=\textwidth]{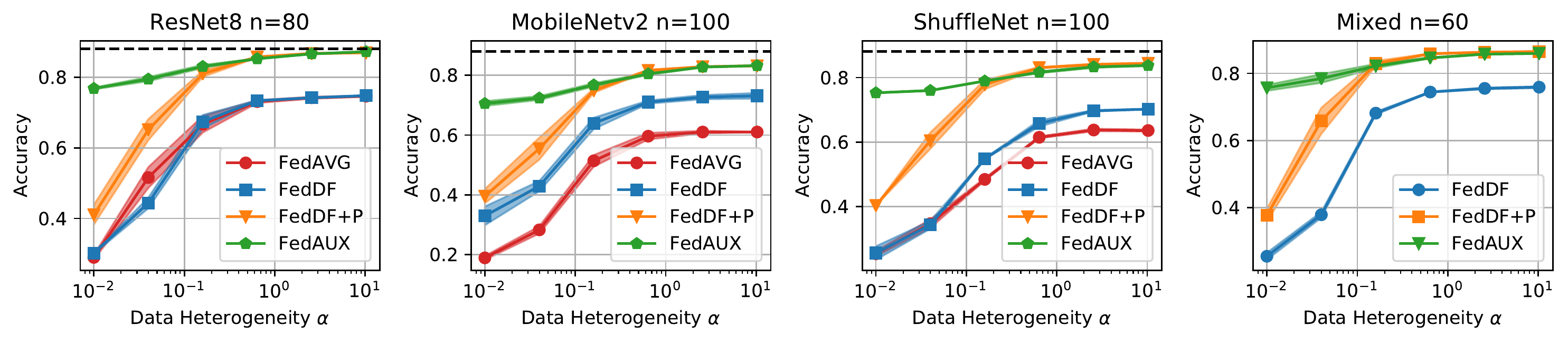}
    \vspace{-0.8cm}
    \caption{Evaluation on \textbf{different neural networks} and client population sizes $n$. Accuracy achieved after $T=100$ communication rounds by different Federated Distillation methods at different levels of data heterogeneity $\alpha$. STL-10 is used as auxiliary data set. In the "Mixed" setting one third of the client population each trains on ResNet8, MobileNetv2 and Shufflenet respectively. Black dashed line indicates centralized training performance.}
    \label{fig:summary_distillation}

    % \subfigure[Evaluating FedAUX on \textbf{NLP Benchmarks}. Performance of FedAUX for different combinations of local datasets and heterogenity levels $\alpha$. 10 clients training TinyBERT at $\alpha=0.01$ and $C=100\%$. Bookcorpus is used as auxiliary data set.]{\includegraphics[width=0.43\textwidth]{images/results_transformer2.pdf}}\hfill
    % \subfigure[\textbf{Privacy Analysis}. Performance of FedAUX for different combinations of the privacy parameters $\varepsilon$, $\delta$ and $\lambda$. 40 clients training Resnet-8 on CIFAR-10 at $\alpha=0.01$ and $C=40\%$. STL-10 is used as auxiliary data set.]{\includegraphics[width=0.53\textwidth]{images/dp_analysis.pdf}}

    \includegraphics[width=0.8\textwidth]{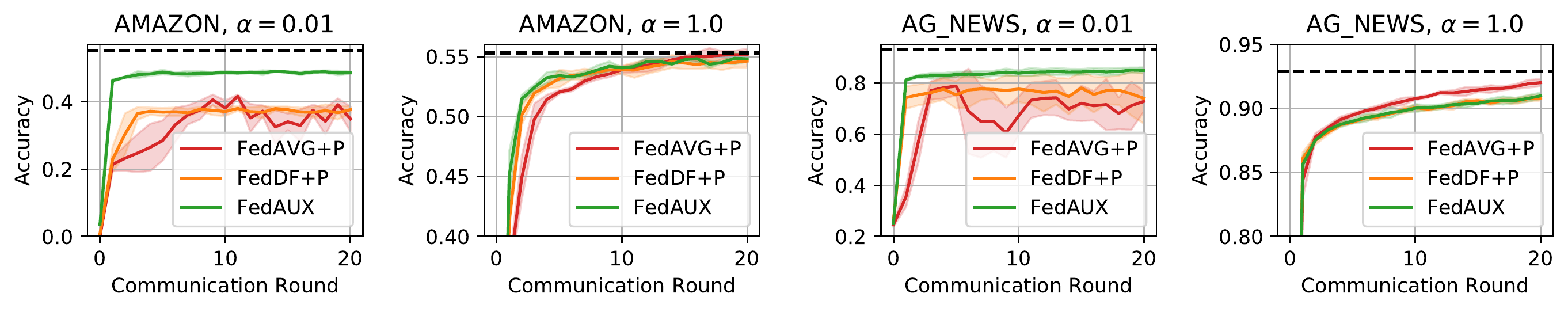}
    \vspace{-0.5cm}
    \caption{Evaluating \textsc{FedAUX} on \textbf{NLP Benchmarks}. Performance of \textsc{FedAUX} for different combinations of local datasets and heterogenity levels $\alpha$. 10 clients training TinyBERT at $\alpha=0.01$ and $C=100\%$. Bookcorpus is used as auxiliary data set. Black dashed line indicates centralized training performance. }
    \label{fig:transformer}
    
    \includegraphics[width=0.8\textwidth]{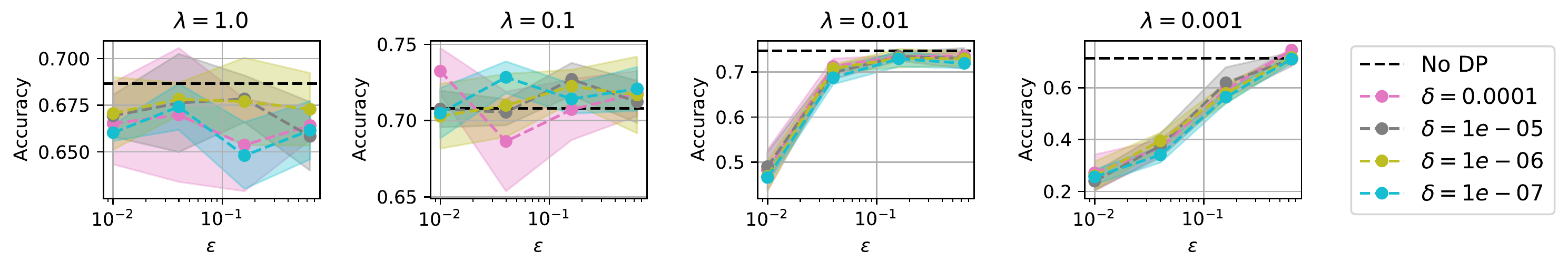}
    \vspace{-0.5cm}
    \caption{\textbf{Privacy Analysis}. Performance of \textsc{FedAUX} for different combinations of the privacy parameters $\varepsilon$, $\delta$ and $\lambda$. 40 clients training Resnet-8 for $T=10$ rounds on CIFAR-10 at $\alpha=0.01$ and $C=40\%$. STL-10 is used as auxiliary data set.}
    \label{fig:dp_analysis}
    
\end{figure*}
% \begin{figure*}[!h]
%     \centering
%     \begin{minipage}{.46\textwidth}
%         \centering
%         \includegraphics[width=\linewidth]{images/results_transformer2.pdf}
%         \vspace{-0.8cm}
%         \caption{$dt=0.1$}
%         \label{fig:prob1_6_2}
%     \end{minipage}%
%     \begin{minipage}{0.46\textwidth}
%         \center
%         \includegraphics[width=\linewidth]{images/dp_analysis.pdf}
%         \vspace{-0.8cm}
%         \caption{$dt =$}
%         \label{fig:prob1_6_1}
        
%         \includegraphics[width=\linewidth]{images/linear_resnet2.pdf}
%         \vspace{-0.8cm}
%         \caption{$dt =$}
%         \label{fig:prob1_6_1}
%     \end{minipage}
% \end{figure*}
\begin{table*}[t!]
    \centering
        \caption{Maximum accuracy achieved by \textsc{FedAUX} and other baseline FL methods after $T=100$ communication rounds, at \textbf{different participation rates $C$} and levels of data heterogeneity $\alpha$. 20 Clients training ResNet-8 on CIFAR-10. Auxiliary data used is STL10. $^*$Methods assume availability of auxiliary data. $^\dagger$Improved Baselines.}
    \label{tab:participation_rate}
    %\resizebox{0.8\textwidth}{!}{%
\begin{tabular}{llllllll}
\toprule
 &        \multicolumn{3}{c}{$\alpha=0.01$} &  & \multicolumn{3}{c}{$\alpha=100.0$} \\
 \cline{2-4}\cline{6-8}
Method & $C=0.2$ &        $C=0.4$ &        $C=0.8$ &        &       $C=0.2$ &       $C=0.4$ &       $C=0.8$ \\
\midrule
\textsc{FedAVG} \macrocite{mcmahan2017communication}    &   19.9$\pm$0.7 &   23.6$\pm$2.0 &   28.9$\pm$2.0 &                &  81.3$\pm$0.1 &  82.2$\pm$0.0 &  82.3$\pm$0.1 \\
\textsc{FedPROX} \macrocite{li2019fedprox}   &   28.4$\pm$2.5 &   34.0$\pm$1.9 &   42.0$\pm$1.0 &                &  81.4$\pm$0.1 &  82.3$\pm$0.2 &  82.0$\pm$0.3 \\
\textsc{FedDF}$^*$ \macrocite{lin2020ensemble}    &   25.0$\pm$0.8 &   27.8$\pm$0.8 &   30.6$\pm$0.3 &                &  80.8$\pm$0.1 &  81.4$\pm$0.3 &  81.5$\pm$0.3 \\
\textsc{FedBE}$^*$ \macrocite{chen2020feddistill}    &   20.9$\pm$0.6 &   25.7$\pm$1.4 &   29.1$\pm$0.1 &                &  81.4$\pm$0.7 &  82.0$\pm$0.1 &  82.2$\pm$0.2 \\
\textsc{FedAVG+P}$^*$$^\dagger$  &   30.4$\pm$7.9 &   32.1$\pm$2.0 &   38.4$\pm$0.5 &                &  89.0$\pm$0.1 &  \textbf{89.5$\pm$0.1} &  \textbf{89.6$\pm$0.1} \\
\textsc{FedPROX+P}$^*$$^\dagger$  &   42.8$\pm$2.7 &   43.1$\pm$0.2 &   49.0$\pm$0.7 &                &  88.9$\pm$0.0 &  89.1$\pm$0.1 &  89.4$\pm$0.0 \\
\textsc{FedDF+P}$^*$$^\dagger$    &   28.8$\pm$3.0 &   39.3$\pm$3.6 &   48.1$\pm$1.1 &                &  88.8$\pm$0.0 &  88.9$\pm$0.1 &  88.9$\pm$0.1 \\
\textsc{FedBE+P}$^*$$^\dagger$    &   30.2$\pm$2.2 &   29.8$\pm$0.8 &   37.7$\pm$0.0 &                &  \textbf{89.1$\pm$0.1} &  89.5$\pm$0.2 &  89.5$\pm$0.0 \\
\textsc{FedAUX}$^*$    &   \textbf{54.2$\pm$0.3} &   \textbf{71.2$\pm$2.1} &   \textbf{78.5$\pm$0.0} &                &  88.9$\pm$0.0 &  89.0$\pm$0.0 &  89.0$\pm$0.1 \\
\bottomrule
\end{tabular}
%}

    \caption{Maximum accuracy achieved by \textsc{FedAUX} and other baseline FL methods after 100 communication rounds, when \textbf{different sets of unlabeled auxiliary data} are used for pre-training and/ or distillation. 40 Clients training ResNet-8 on CIFAR-10 at $C=40\%$.}
    \label{tab:aux_data}
    %\resizebox{0.5\textwidth}{!}{%
\begin{tabular}{llrrrrrrrr}
\toprule
& & \multicolumn{8}{c}{Auxiliary Data}\\
\cline{3-10}
$\alpha$ & Method &  STL-10 &  CIFAR-100 &   SVHN &  Invertebr. &  Birds &  Devices &   Dogs &  Structures  \\
\midrule
0.01   & \textsc{FedDF} &           27.9$\pm$3.2 &           29.5$\pm$6.2 &           28.1$\pm$3.9 &           28.5$\pm$3.6 &           30.1$\pm$2.0 &           26.3$\pm$0.2 &           28.9$\pm$5.1 &           30.2$\pm$7.0 \\
       & \textsc{FedDF+P} &           43.0$\pm$5.2 &           41.6$\pm$1.1 &           29.6$\pm$3.4 &           38.8$\pm$6.5 &           41.4$\pm$5.9 &           35.9$\pm$4.9 &           41.1$\pm$7.3 &           36.7$\pm$7.1 \\
       & \textsc{FedAUX} &  \textbf{76.8$\pm$0.9} &  \textbf{71.5$\pm$2.5} &  \textbf{43.7$\pm$1.5} &  \textbf{68.2$\pm$0.7} &  \textbf{65.7$\pm$3.1} &  \textbf{71.5$\pm$0.1} &  \textbf{71.8$\pm$3.8} &  \textbf{64.1$\pm$3.3} \\
       \midrule
100.00 & \textsc{FedDF} &           79.3$\pm$0.7 &           79.9$\pm$0.1 &           80.9$\pm$0.1 &           80.2$\pm$0.1 &           80.2$\pm$0.4 &           79.4$\pm$0.3 &           79.7$\pm$0.4 &           80.1$\pm$0.2 \\
       & \textsc{FedDF+P} &           88.3$\pm$0.0 &           86.7$\pm$0.0 &  \textbf{81.7$\pm$0.2} &           87.4$\pm$0.1 &           87.6$\pm$0.0 &           87.7$\pm$0.1 &           88.4$\pm$0.0 &  \textbf{87.4$\pm$0.1} \\
       & \textsc{FedAUX} &  \textbf{88.5$\pm$0.0} &  \textbf{86.7$\pm$0.1} &           81.6$\pm$0.0 &  \textbf{87.8$\pm$0.1} &  \textbf{87.8$\pm$0.1} &  \textbf{87.8$\pm$0.0} &  \textbf{88.6$\pm$0.0} &           87.3$\pm$0.1 \\
\bottomrule
\end{tabular}

    \caption{\textbf{One-shot performance} of different FL methods. Maximum accuracy achieved after $T=1$ communication rounds at participation-rate $C=100\%$. Each client trains for $E=40$ local epochs.}
    \label{tab:one-shot}
\begin{tabular}{lllllcllll}
\toprule
 & \multicolumn{4}{c}{MobileNetv2, $n=100$} & & \multicolumn{4}{c}{Shufflenet, $n=100$} \\
\cline{2-5}
\cline{7-9}
Method &            $\alpha=0.01$ &   $\alpha=0.04$     &    $\alpha=0.16$ &           $\alpha=10.24$ &      &      $\alpha=0.01$ &   $\alpha=0.04$     &          $\alpha=0.16$ &           $\alpha=10.24$ \\
\midrule
\textsc{FedAVG}    &           10.3$\pm$0.0 &           13.6$\pm$2.3 &           23.6$\pm$0.0 &           30.5$\pm$0.9 &            &           12.1$\pm$0.8 &           17.4$\pm$0.4 &           28.2$\pm$0.8 &           37.8$\pm$0.7 \\
\textsc{FedPROX}   &  11.6$\pm$0.8 &  14.3$\pm$1.4 &  23.7$\pm$0.3 &  30.5$\pm$0.5 &   &  12.9$\pm$1.7 &  18.9$\pm$0.2 &  29.4$\pm$0.3 &  38.9$\pm$0.5 \\
\textsc{FedDF}     &           16.8$\pm$4.2 &           29.5$\pm$3.8 &           37.7$\pm$1.1 &           40.4$\pm$0.5 &            &           16.0$\pm$5.1 &           27.3$\pm$0.1 &           38.7$\pm$0.2 &           45.5$\pm$0.5 \\
\textsc{FedAVG+P}  &           24.3$\pm$1.1 &           44.0$\pm$4.4 &           57.6$\pm$3.7 &           69.9$\pm$0.0 &            &           25.5$\pm$1.4 &           44.2$\pm$0.1 &           62.9$\pm$1.6 &           71.9$\pm$0.1 \\
\textsc{FedPROX+P} &  27.2$\pm$2.2 &  43.4$\pm$3.6 &  56.9$\pm$3.9 &  70.0$\pm$0.1 &   &  28.4$\pm$0.2 &  47.1$\pm$1.5 &  63.3$\pm$1.2 &  71.9$\pm$0.1 \\
\textsc{FedDF+P}   &           46.7$\pm$5.6 &           61.1$\pm$1.3 &           67.6$\pm$0.5 &           71.2$\pm$0.1 &            &           40.4$\pm$2.7 &           59.4$\pm$0.8 &           68.8$\pm$0.2 &           72.7$\pm$0.0 \\
\textsc{FedAUX}    &  \textbf{64.8$\pm$0.0} &  \textbf{65.5$\pm$1.0} &  \textbf{68.2$\pm$0.2} &  \textbf{71.3$\pm$0.1} &            &  \textbf{66.9$\pm$0.6} &  \textbf{68.6$\pm$0.4} &  \textbf{70.8$\pm$0.3} &  \textbf{72.9$\pm$0.1} \\
\bottomrule
\end{tabular}
\end{table*}
\textbf{Federated Learning environment and Data Partitioning}: We consider Federated Learning problems with up to $n=100$ participating clients. In all experiments, we split the training data evenly among the clients according to a dirichlet distribution following the procedure outlined in \macrocite{hsu2019measuring} and illustrated in Fig. \ref{fig:alpha}. This allows us to smoothly adapt the level of non-iid-ness in the client data using the dirichlet parameter $\alpha$. We experiment with values for $\alpha$ varying between 100.0 and 0.01. A value of $\alpha=100.0$ results in an almost identical label distribution, while setting $\alpha=0.01$ results in a split, where the vast majority of data on every client stems from one single class. See Appendix \ref{supp:data_splitting} for a more detailed description of our data splitting procedure. We vary the client participation rate $C$ in every round between 20\% and 100\%. 
%The validation data in each case is following the distribution of the client training data, as is standard convention in FL.\\

\textbf{Pre-training strategy:} For our image classification problems, we use contrastive representation learning as described in \macrocite{chen2020simple} for pre-training. We use the default set of data augmentations proposed in the paper and train with the Adam optimizer, learning rate set to $10^{-3}$ and a batch-size of 512. For our text classification problems, we pre-train using self-supervised next-word prediction. 

\textbf{Training the Scoring model and Privacy Setting:} We set the default privacy parameters to $\lambda=0.1$, $\varepsilon=0.1$ and $\delta=1e-5$ respectively and solve \eqref{eq:ERM} by running L-BFGS \macrocite{liu1989limited} until convergence ($\leq 1000$ steps). 

\textbf{Baselines:}  We compare the performance of \textsc{FedAUX} to state-of-the-art FL methods: \textsc{FedAVG} \macrocite{mcmahan2017communication}, \textsc{FedProx} \macrocite{li2019fedprox}, Federated Ensemble Distillation (\textsc{FedDF}) \macrocite{lin2020ensemble} and \textsc{FedBE} \macrocite{chen2020feddistill}. To clearly discern the performance benefits of the two components of \textsc{FedAUX} (unsupervised pre-training and weighted ensemble distillation), we also report performance metrics on versions of these methods where the auxiliary data was used to pre-train the feature extractor $h$ ("\textsc{FedAVG+P}", "\textsc{FedProx+P}", "\textsc{FedDF+P}"  resp. "\textsc{FedBE+P}"). For \textsc{FedBE} we set the sample size to 10 as suggested in the paper. For \textsc{FedProx} we always tune the proximal parameter $\mu$.

\textbf{Optimization:} On all image classification task, we use the very popular Adam optimizer \macrocite{kingma2014adam}, with a fixed learning rate of $\eta=10^{-3}$ and a batch-size of 32 for local training. Distillation is performed for one epoch for all methods using Adam at a batch-size of 128 and fixed learning rate of $5e-5$. More detailed hyperparameter analysis in Appendix \ref{supp:hyperparameter} shows that this choice of optimization parameters is approximately optimal for all of the methods. If not stated otherwise, the number of local epochs $E$  is set to 1.

\subsection{Evaluating \textsc{FedAUX} on common Federated Learning Benchmarks}
We start out by evaluating the performance of \textsc{FedAUX} on classic benchmarks for Federated image classification. Figure \ref{fig:summary_distillation} shows the maximum accuracy achieved by different Federated Distillation methods after $T=100$ communication rounds at different levels of data heterogeneity. As we can see, \textsc{FedAUX} distinctively outperforms \textsc{FedDF} on the entire range of data heterogeneity levels $\alpha$ on all benchmarks. For instance, when training ResNet8 with $n=80$ clients at $\alpha=0.01$, \textsc{FedAUX} \emph{raises the maximum achieved accuracy from 18.2\% to 78.1\%} (under the same set of assumptions).  The two components of \textsc{FedAUX}, unsupervised pre-training and weighted ensemble distillation, both contribute independently to the performance improvement, as can be seen when comparing with \textsc{FedDF+P}, which only uses unsupervised pre-training. Weighted ensemble distillation as done in \textsc{FedAUX} leads to greater or equal performance than equally weighted distillation (\textsc{FedDF+P}) across all levels of data heterogeneity. The same overall picture can be observed in the "Mixed" setting where clients train different model architectures. Detailed training curves are given in the Appendix \ref{supp:training_curves}.

Table \ref{tab:participation_rate} compares the performance of \textsc{FedAUX} and baseline methods at different client participation rates $C$. We can see that \textsc{FedAUX} benefits from higher participation rates. In all scenarios, methods which are initialized using the pre-trained feature-extractor $h_0$ distinctively outperform their randomly initialized counterparts. In the iid setting at $\alpha=100.0$ \textsc{FedAUX} is mostly en par with the (improved) parameter averaging based methods \textsc{FedAVG+P} and \textsc{FedPROX+P}, with a maximum performance gap of 0.8\%. At $\alpha=0.01$ on the other hand \textsc{FedAUX} outperforms all other methods with a margin of up to 29\%.

\subsection{Evaluating \textsc{FedAUX} on NLP Benchmarks} 
%Analyzing \textsc{FedAUX} in the NLP domain is particulary interesting since transformer-based architectures, which we will use here, are getting increasingly common \macrocite{devlin2018bert,radford2019language,brown2020language,yang2020xlnet} and NLP as a field deals with sparse data and therefore higher heterogeneity in local data sets only increases potential detrimental effects \todo{find proof for this claim}. 
%For our NLP experiments we finetune TinyBERT \macrocite{jiao2020tinybert}, which achieves comparable accuracy on different NLP benchmarks to contending models such as DistilBERT \macrocite{sanh2020distilbert} while being seven times smaller than its teacher model BERT \macrocite{devlin2018bert}.
Figure \ref{fig:transformer} shows learning curves for Federated training of TinyBERT on the Amazon and AG-News datasets at two different levels of data heterogeneity $\alpha$. We observe, that \textsc{FedAUX} significantly outperforms \textsc{FedDF+P} as well as \textsc{FedAVG+P} in the heterogeneous setting ($\alpha=0.01$) and reaches 95\% of its final accuracy after one communication round on both datasets, indicating suitability for one-shot learning. On more homogeneous data ($\alpha=1.0$) \textsc{FedAUX} performs mostly en par with pre-trained versions of \textsc{FedAVG} and \textsc{FedDF}, with a maximal performance gap of 1.1 \% accuracy on the test set.
We note, that effects of data heterogeneity are less severe as in this setting as both the AG News and the Amazon data set only have four and five labels respectively and an $\alpha$ of $1.0$ already leads to a distribution where each clients owns a subset of the private data set containing all possible labels. Further details on our implementation can be found the Appendix \ref{supp:transformer_implementation}.

\subsection{Privacy Analysis of \textsc{FedAUX}}
\label{sec:ex_privacy}
Figure \ref{fig:dp_analysis} examines the dependence of \textsc{FedAUX}' training performance of the privacy parameters $\varepsilon$, $\delta$ and the regularization parameter $\lambda$. As we can see, performance comparable to non-private scoring is achievable at conservative privacy parameters $\varepsilon$, $\delta$. For instance, at $\lambda=0.01$ setting $\varepsilon=0.04$ and $\delta=10^{-6}$ reduces the accuracy from 74.6\% to 70.8\%. At higher values of $\lambda$, better privacy guarantees have an even less harmful effect, at the cost however of an overall degradation in performance. Throughout this empirical study, we have set the default privacy parameters to $\lambda=0.1$, $\varepsilon=0.1$ and $\delta=1e-5$. 
We also perform an empirical privacy analysis in the Appendix \ref{supp:epirical_privacy}, which provides additional intuitive understanding and confidence in the privacy properties of our method.

\subsection{Evaluating the dependence on Auxiliary Data}
\label{sec:auxiliary_data}
Next, we investigate the influence of the auxiliary data set $D_{aux}$ on unsupervised pretraining, distillation and weighted distillation respectively. We use CIFAR-10 as training data set and consider 8 different auxiliary data sets, which differ w.r.t their similarity to this client training data - from more similar (STL-10, CIFAR-100) to less similar (Devices, SVHN)\footnote{The CIFAR-10 data set contains images from the classes airplane, automobile, bird, cat, deer, dog, frog, horse, ship and truc.}. 
Table \ref{tab:aux_data} shows the maximum achieved accuracy after $T=100$ rounds when each of these data sets is used as auxiliary data. As we can see, performance \emph{always} improves when auxiliary data is used for unsupervised pre-training. Even for the highly dissimilar SVHN data set (which contains images of house numbers) performance of \textsc{FedDF+P} improves by 1\% over \textsc{FedDF} in both the iid and non-iid regime.  For other data sets like Dogs, Birds or Invertebrates performance improves by up to 14\%, although they overlap with only one single class of the CIFAR-10 data set. The outperformance of \textsc{FedAUX} on such a wide variety of highly dissimilar data sets suggest that beneficial auxiliary data should be available in the majority of practical FL problems and also has positive implications from the perspective of privacy.  Interestingly, performance of \textsc{FedDF} seems to only weakly correlate with the performance of \textsc{FedDF+P} and \textsc{FedAUX} as a function of the auxiliary data set. This suggests, that the properties, which make a data set useful for distillation are not the same ones that make it useful for pre-training and weighted distillation. Investigating this relationship further is an interesting direction of future research.

\begin{figure}[t!]
    \centering
    \includegraphics[width=0.5\textwidth]{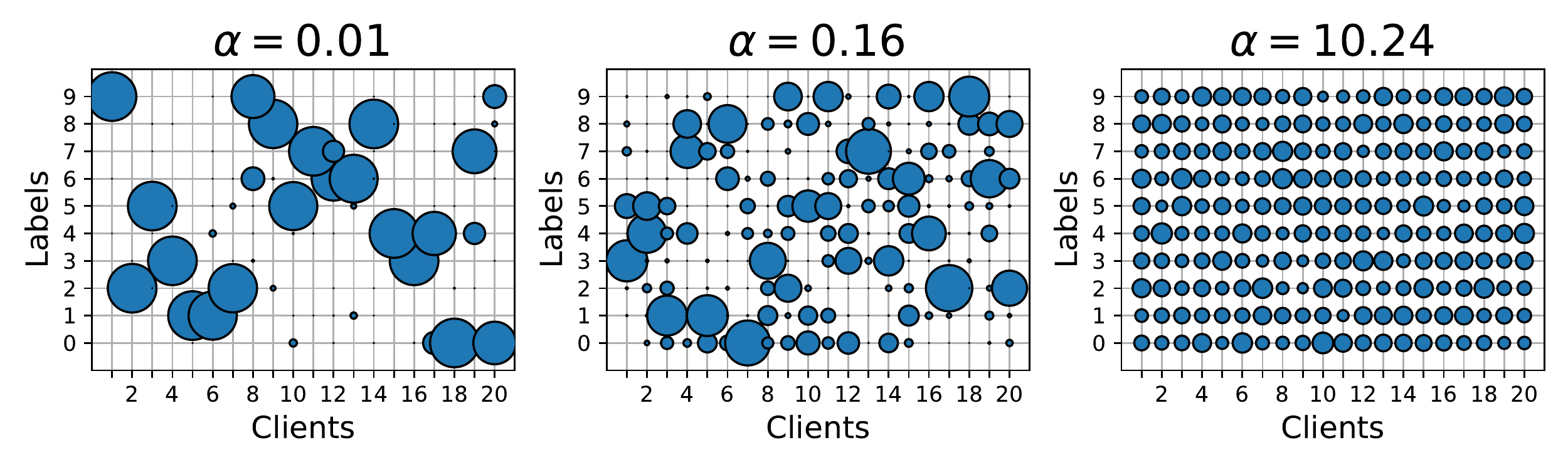}
    \vspace{-0.8cm}
    \caption{Illustration of the \textbf{Dirichlet data splitting strategy} we use throughout the paper, exemplary for a Federated Learning setting with 20 Clients and 10 different classes. Marker size indicates the number of samples held by one client for each particular class. Lower values of $\alpha$ lead to more heterogeneous distributions of client data. Figure adapted from \macrocite{lin2020ensemble}.}
    \label{fig:alpha}
    
    \includegraphics[width=0.5\textwidth]{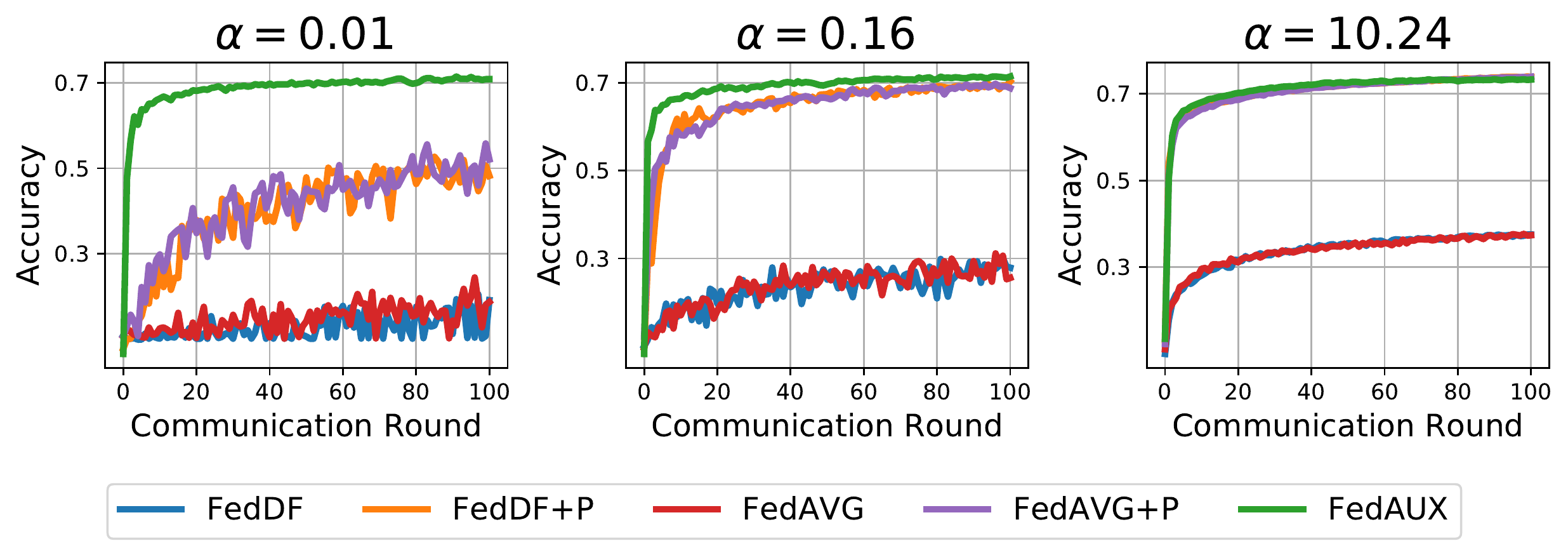}
    \vspace{-0.8cm}
    \caption{\textbf{Linear evaluation}. Training curves for different Federated Learning methods at different levels of data heterogeneity $\alpha$ when only the classification head $g$ is updated in the training phase. A total of $n=80$ clients training ResNet8 on CIFAR-10 at $C=40\%$, using STL-10 as auxiliary data set.}
    \label{fig:linear}
\end{figure}

\subsection{\textsc{FedAUX} in hardware-constrained settings}
\textbf{Linear Evaluation:}
In settings where the FL clients are hardware-constrained mobile or IoT devices, local training of entire deep neural networks like ResNet8 might be infeasible. We therefore also consider the evaluation of different FL methods, when only the linear classification head $g$ is updated during the training phase. Figure \ref{fig:linear} shows training curves in this setting when clients hold data from the CIFAR-10 data set. We see that in this setting performance of \textsc{FedAUX} is high, independent of the data heterogeneity levels $\alpha$, suggesting that in the absence of non-convex training dynamics our proposed scoring method actually yields robust weighted ensembles in the sense of \macrocite{mansour2008domain}. We note, that \textsc{FedAUX} also trains much more smoothly, than all other baseline methods. 

\textbf{One-Shot Evaluation:}
In many FL applications, the number of times a client can participate in the Federated training is restricted by communication, energy and/ or privacy constraints \macrocite{guha2019one-shot, papernot2018scalable}. To study these types of settings, we investigate the performance of \textsc{FedAUX} and other FL methods in Federated one-shot learning where we set $T=1$ and $C=100\%$. Table \ref{tab:one-shot} compares performance in this setting for $n=100$ clients training MobileNetv2 resp. ShuffleNet. \textsc{FedAUX} outperforms the baseline methods in this setting at all levels of data heterogeneity $\alpha$.

\section{Conclusion}
In this work, we explored Federated Learning in the presence of unlabeled auxiliary data, an assumption made in the quickly growing area of Federated Distillation. By leveraging auxiliary data for unsupervised pre-training and weighted ensemble distillation we were able to demonstrate that this assumption is rather strong and can lead to drastically improved performance of FL algorithms. These results reveal the limited merit in comparing FD based methods with parameter averaging based methods (which do not make this assumption) and thus have implications for the future evaluation of FD methods in general.

\mainrefs
% \bibliography{references}

\begin{thebibliography}{69}
\providecommand{\natexlab}[1]{#1}
\providecommand{\url}[1]{\texttt{#1}}
\expandafter\ifx\csname urlstyle\endcsname\relax
  \providecommand{\doi}[1]{doi: #1}\else
  \providecommand{\doi}{doi: \begingroup \urlstyle{rm}\Url}\fi

\bibitem[Abadi et~al.(2016)Abadi, Chu, Goodfellow, McMahan, Mironov, Talwar,
  and Zhang]{abadi2016deep}
Abadi, M., Chu, A., Goodfellow, I., McMahan, H.~B., Mironov, I., Talwar, K.,
  and Zhang, L.
\newblock Deep learning with differential privacy.
\newblock In \emph{Proceedings of the 2016 {ACM} {SIGSAC} Conference on
  Computer and Communications Security {(CCS)}}, pp.\  308--318, 2016.

\bibitem[Ahn et~al.(2019)Ahn, Simeone, and Kang]{ahn2019wireless}
Ahn, J.-H., Simeone, O., and Kang, J.
\newblock Wireless federated distillation for distributed edge learning with
  heterogeneous data.
\newblock In \emph{2019 IEEE 30th Annual International Symposium on Personal,
  Indoor and Mobile Radio Communications (PIMRC)}, pp.\  1--6. IEEE, 2019.

\bibitem[Anil et~al.(2018)Anil, Pereyra, Passos, Ormandi, Dahl, and
  Hinton]{anil2018large}
Anil, R., Pereyra, G., Passos, A., Ormandi, R., Dahl, G.~E., and Hinton, G.~E.
\newblock Large scale distributed neural network training through online
  distillation.
\newblock \emph{arXiv preprint arXiv:1804.03235}, 2018.

\bibitem[Ben{-}David et~al.(2010)Ben{-}David, Blitzer, Crammer, Kulesza,
  Pereira, and Vaughan]{bendavid2010domain}
Ben{-}David, S., Blitzer, J., Crammer, K., Kulesza, A., Pereira, F., and
  Vaughan, J.~W.
\newblock A theory of learning from different domains.
\newblock \emph{Mach. Learn.}, 79\penalty0 (1-2):\penalty0 151--175, 2010.

\bibitem[Bucila et~al.(2006)Bucila, Caruana, and
  Niculescu{-}Mizil]{bucila2006compression}
Bucila, C., Caruana, R., and Niculescu{-}Mizil, A.
\newblock Model compression.
\newblock In \emph{Proceedings of the 12th {ACM} {SIGKDD} International
  Conference on Knowledge Discovery and Data Mining {(KDD)}}, pp.\  535--541,
  2006.

\bibitem[Chang et~al.(2019)Chang, Shejwalkar, Shokri, and
  Houmansadr]{chang2019cronus}
Chang, H., Shejwalkar, V., Shokri, R., and Houmansadr, A.
\newblock Cronus: Robust and heterogeneous collaborative learning with
  black-box knowledge transfer.
\newblock \emph{arXiv preprint arXiv:1912.11279}, 2019.

\bibitem[Chaudhuri et~al.(2011)Chaudhuri, Monteleoni, and
  Sarwate]{chaudhuri2011differentially}
Chaudhuri, K., Monteleoni, C., and Sarwate, A.~D.
\newblock Differentially private empirical risk minimization.
\newblock \emph{J. Mach. Learn. Res.}, 12:\penalty0 1069--1109, 2011.

\bibitem[Chen \& Chao(2020)Chen and Chao]{chen2020feddistill}
Chen, H.-Y. and Chao, W.-L.
\newblock {FedDistill: M}aking bayesian model ensemble applicable to federated
  learning.
\newblock \emph{arXiv preprint arXiv:2009.01974}, 2020.

\bibitem[Chen et~al.(2020)Chen, Kornblith, Norouzi, and Hinton]{chen2020simple}
Chen, T., Kornblith, S., Norouzi, M., and Hinton, G.~E.
\newblock A simple framework for contrastive learning of visual
  representations.
\newblock In \emph{Proceedings of the 37th International Conference on Machine
  Learning {(ICML)}}, pp.\  1597--1607, 2020.

\bibitem[Deng et~al.(2009)Deng, Dong, Socher, Li, Li, and
  Fei-Fei]{deng2009imagenet}
Deng, J., Dong, W., Socher, R., Li, L.-J., Li, K., and Fei-Fei, L.
\newblock {ImageNet}: {A} large-scale hierarchical image database.
\newblock In \emph{Proceedings of the {IEEE} Computer Society Conference on
  Computer Vision and Pattern Recognition {(CVPR)}}, pp.\  248--255, 2009.

\bibitem[Devlin et~al.(2019)Devlin, Chang, Lee, and Toutanova]{devlin2018bert}
Devlin, J., Chang, M., Lee, K., and Toutanova, K.
\newblock {BERT: P}re-training of deep bidirectional transformers for language
  understanding.
\newblock In \emph{Proceedings of the 2019 Conference of the North American
  Chapter of the Association for Computational Linguistics: Human Language
  Technologies, {(NAACL-HLT)}}, volume~1, pp.\  4171--4186, 2019.

\bibitem[Dwork \& Roth(2014)Dwork and Roth]{dwork2014algorithmic}
Dwork, C. and Roth, A.
\newblock The algorithmic foundations of differential privacy.
\newblock \emph{Found. Trends Theor. Comput. Sci.}, 9\penalty0 (3-4):\penalty0
  211--407, 2014.

\bibitem[Ghosh et~al.(2019)Ghosh, Hong, Yin, and Ramchandran]{ghosh2019noniid}
Ghosh, A., Hong, J., Yin, D., and Ramchandran, K.
\newblock Robust federated learning in a heterogeneous environment.
\newblock \emph{arXiv preprint arXiv:1906.06629}, 2019.

\bibitem[Guha et~al.(2019)Guha, Talwalkar, and Smith]{guha2019one-shot}
Guha, N., Talwalkar, A., and Smith, V.
\newblock One-shot federated learning.
\newblock \emph{arXiv preprint arXiv:1902.11175}, 2019.

\bibitem[Hashem \& Schmeiser(1993)Hashem and Schmeiser]{hashem1993ensemble}
Hashem, S. and Schmeiser, B.
\newblock Approximating a function and its derivatives using mse-optimal linear
  combinations of trained feedforward neural networks.
\newblock In \emph{Proceedings of the World Congress on Neural Networks},
  volume~1, pp.\  617--620, 1993.

\bibitem[He et~al.(2016)He, Zhang, Ren, and Sun]{he2016deep}
He, K., Zhang, X., Ren, S., and Sun, J.
\newblock Deep residual learning for image recognition.
\newblock In \emph{Proceedings of the IEEE Conference on Computer Vision and
  Pattern Recognition {(CVPR)}}, pp.\  770--778, 2016.

\bibitem[Hinton et~al.(2015)Hinton, Vinyals, and Dean]{hinton2015distill}
Hinton, G., Vinyals, O., and Dean, J.
\newblock Distilling the knowledge in a neural network.
\newblock \emph{arXiv preprint arXiv:1503.02531}, 2015.

\bibitem[Hoffman et~al.(2018)Hoffman, Mohri, and Zhang]{hoffman2018domain}
Hoffman, J., Mohri, M., and Zhang, N.
\newblock Algorithms and theory for multiple-source adaptation.
\newblock In \emph{Advances in Neural Information Processing Systems
  {(NeurIPS)}}, volume~31, pp.\  8256--8266, 2018.

\bibitem[Hsu et~al.(2019)Hsu, Qi, and Brown]{hsu2019measuring}
Hsu, T.-M.~H., Qi, H., and Brown, M.
\newblock Measuring the effects of non-identical data distribution for
  federated visual classification.
\newblock \emph{arXiv preprint arXiv:1909.06335}, 2019.

\bibitem[Itahara et~al.(2020)Itahara, Nishio, Koda, Morikura, and
  Yamamoto]{itahara2020distill}
Itahara, S., Nishio, T., Koda, Y., Morikura, M., and Yamamoto, K.
\newblock Distillation-based semi-supervised federated learning for
  communication-efficient collaborative training with non-iid private data.
\newblock \emph{arXiv preprint arXiv:2008.06180}, 2020.

\bibitem[Jeong et~al.(2018)Jeong, Oh, Kim, Park, Bennis, and
  Kim]{jeong2018distill}
Jeong, E., Oh, S., Kim, H., Park, J., Bennis, M., and Kim, S.
\newblock Communication-efficient on-device machine learning: Federated
  distillation and augmentation under non-iid private data.
\newblock \emph{arXiv preprint arXiv:1811.11479}, 2018.

\bibitem[Jeong et~al.(2020)Jeong, Yoon, Yang, and Hwang]{jeong2020federated}
Jeong, W., Yoon, J., Yang, E., and Hwang, S.~J.
\newblock Federated semi-supervised learning with inter-client consistency.
\newblock \emph{arXiv preprint arXiv:2006.12097}, 2020.

\bibitem[Jiao et~al.(2020)Jiao, Yin, Shang, Jiang, Chen, Li, Wang, and
  Liu]{jiao2020tinybert}
Jiao, X., Yin, Y., Shang, L., Jiang, X., Chen, X., Li, L., Wang, F., and Liu,
  Q.
\newblock {TinyBERT: D}istilling {BERT} for natural language understanding.
\newblock In \emph{Proceedings of the 2020 Conference on Empirical Methods in
  Natural Language Processing: Findings {(EMNLP)}}, pp.\  4163--4174, 2020.

\bibitem[Jim{\'e}nez(1998)]{jimenez1998dynamically}
Jim{\'e}nez, D.
\newblock Dynamically weighted ensemble neural networks for classification.
\newblock In \emph{IEEE International Joint Conference on Neural Networks
  Proceedings. IEEE World Congress on Computational Intelligence}, volume~1,
  pp.\  753--756, 1998.

\bibitem[Keung et~al.(2020)Keung, Lu, Szarvas, and Smith]{marc_reviews}
Keung, P., Lu, Y., Szarvas, G., and Smith, N.~A.
\newblock The multilingual amazon reviews corpus.
\newblock In \emph{Proceedings of the 2020 Conference on Empirical Methods in
  Natural Language Processing {(EMNLP)}}, pp.\  4563--4568, 2020.

\bibitem[Kingma \& Ba(2014)Kingma and Ba]{kingma2014adam}
Kingma, D.~P. and Ba, J.
\newblock Adam: A method for stochastic optimization.
\newblock \emph{arXiv preprint arXiv:1412.6980}, 2014.

\bibitem[Li \& Wang(2019)Li and Wang]{li2019fedmd}
Li, D. and Wang, J.
\newblock {FedMD}: {H}eterogenous federated learning via model distillation.
\newblock \emph{arXiv preprint arXiv:1910.03581}, 2019.

\bibitem[Li et~al.(2019)Li, Wen, and He]{li2019federated}
Li, Q., Wen, Z., and He, B.
\newblock Federated learning systems: Vision, hype and reality for data privacy
  and protection.
\newblock \emph{arXiv preprint arXiv:1907.09693}, 2019.

\bibitem[Li et~al.(2020{\natexlab{a}})Li, Sahu, Zaheer, Sanjabi, Talwalkar, and
  Smith]{li2019fedprox}
Li, T., Sahu, A.~K., Zaheer, M., Sanjabi, M., Talwalkar, A., and Smith, V.
\newblock Federated optimization in heterogeneous networks.
\newblock In \emph{Proceedings of Machine Learning and Systems {(MLSys)}},
  2020{\natexlab{a}}.

\bibitem[Li et~al.(2020{\natexlab{b}})Li, Huang, Yang, Wang, and
  Zhang]{li2020covergence}
Li, X., Huang, K., Yang, W., Wang, S., and Zhang, Z.
\newblock On the convergence of {FedAvg} on non-iid data.
\newblock In \emph{Proceedings of 8th International Conference on Learning
  Representations {(ICLR)}}. OpenReview.net, 2020{\natexlab{b}}.

\bibitem[Li et~al.(2021)Li, Zhou, Wang, Mi, and Hospedales]{li2021fedh2l}
Li, Y., Zhou, W., Wang, H., Mi, H., and Hospedales, T.~M.
\newblock Fedh2l: Federated learning with model and statistical heterogeneity.
\newblock \emph{arXiv preprint arXiv:2101.11296}, 2021.

\bibitem[Lin et~al.(2020)Lin, Kong, Stich, and Jaggi]{lin2020ensemble}
Lin, T., Kong, L., Stich, S.~U., and Jaggi, M.
\newblock Ensemble distillation for robust model fusion in federated learning.
\newblock In \emph{Advances in Neural Information Processing Systems
  {(NeurIPS)}}, volume~33, 2020.

\bibitem[Liu \& Nocedal(1989)Liu and Nocedal]{liu1989limited}
Liu, D.~C. and Nocedal, J.
\newblock On the limited memory {BFGS} method for large scale optimization.
\newblock \emph{Math. Program.}, 45\penalty0 (1-3):\penalty0 503--528, 1989.

\bibitem[Mansour et~al.(2008)Mansour, Mohri, and
  Rostamizadeh]{mansour2008domain}
Mansour, Y., Mohri, M., and Rostamizadeh, A.
\newblock Domain adaptation with multiple sources.
\newblock In \emph{Advances in Neural Information Processing Systems
  {(NeurIPS)}}, volume~21, pp.\  1041--1048, 2008.

\bibitem[Mansour et~al.(2020)Mansour, Mohri, Ro, and
  Suresh]{mansour2020personalization}
Mansour, Y., Mohri, M., Ro, J., and Suresh, A.~T.
\newblock Three approaches for personalization with applications to federated
  learning.
\newblock \emph{arXiv preprint arXiv:2002.10619}, 2020.

\bibitem[Masoudnia \& Ebrahimpour(2014)Masoudnia and
  Ebrahimpour]{masoudnia2014mixture}
Masoudnia, S. and Ebrahimpour, R.
\newblock Mixture of experts: {A} literature survey.
\newblock \emph{Artif. Intell. Rev.}, 42\penalty0 (2):\penalty0 275--293, 2014.

\bibitem[McMahan et~al.(2017)McMahan, Moore, Ramage, Hampson, and
  y~Arcas]{mcmahan2017communication}
McMahan, B., Moore, E., Ramage, D., Hampson, S., and y~Arcas, B.~A.
\newblock Communication-efficient learning of deep networks from decentralized
  data.
\newblock In \emph{Proceedings of the 20th International Conference on
  Artificial Intelligence and Statistics {(AISTATS)}}, pp.\  1273--1282, 2017.

\bibitem[Merity et~al.(2016)Merity, Xiong, Bradbury, and
  Socher]{merity2016pointer}
Merity, S., Xiong, C., Bradbury, J., and Socher, R.
\newblock Pointer sentinel mixture models.
\newblock \emph{arXiv preprint arXiv:1609.07843}, 2016.

\bibitem[Mohri et~al.(2019)Mohri, Sivek, and Suresh]{mohri2019agnostic}
Mohri, M., Sivek, G., and Suresh, A.~T.
\newblock Agnostic federated learning.
\newblock In \emph{Proceedings of the 36th International Conference on Machine
  Learning {(ICML)}}, pp.\  4615--4625, 2019.

\bibitem[Mothukuri et~al.(2021)Mothukuri, Parizi, Pouriyeh, Huang,
  Dehghantanha, and Srivastava]{mothukuri2020survey}
Mothukuri, V., Parizi, R.~M., Pouriyeh, S., Huang, Y., Dehghantanha, A., and
  Srivastava, G.
\newblock A survey on security and privacy of federated learning.
\newblock \emph{Future Gener. Comput. Syst.}, 115:\penalty0 619--640, 2021.

\bibitem[Nayak et~al.(2019)Nayak, Mopuri, Shaj, Radhakrishnan, and
  Chakraborty]{nayak2019zero}
Nayak, G.~K., Mopuri, K.~R., Shaj, V., Radhakrishnan, V.~B., and Chakraborty,
  A.
\newblock Zero-shot knowledge distillation in deep networks.
\newblock In \emph{Proceedings of the 36th International Conference on Machine
  Learning, {(ICML)}}, pp.\  4743--4751, 2019.

\bibitem[Opitz \& Maclin(1999)Opitz and Maclin]{opitz1999popular}
Opitz, D.~W. and Maclin, R.
\newblock Popular ensemble methods: {A}n empirical study.
\newblock \emph{J. Artif. Intell. Res.}, 11:\penalty0 169--198, 1999.

\bibitem[Papernot et~al.(2018)Papernot, Song, Mironov, Raghunathan, Talwar, and
  Erlingsson]{papernot2018scalable}
Papernot, N., Song, S., Mironov, I., Raghunathan, A., Talwar, K., and
  Erlingsson, {\'{U}}.
\newblock Scalable private learning with {PATE}.
\newblock In \emph{Proceedings of the 6th International Conference on Learning
  Representations {(ICLR)}}. OpenReview.net, 2018.

\bibitem[Perrone \& Cooper(1993)Perrone and Cooper]{perrone1993ensemble}
Perrone, M.~P. and Cooper, L.~N.
\newblock When networks disagree: {E}nsemble methods for hybrid neural
  networks.
\newblock In Mammone, R.~J. (ed.), \emph{Neural Networks for Speech and Image
  Processing}. Chapman and Hall, 1993.

\bibitem[Radford et~al.(2019)Radford, Wu, Child, Luan, Amodei, and
  Sutskever]{radford2019language}
Radford, A., Wu, J., Child, R., Luan, D., Amodei, D., and Sutskever, I.
\newblock Language models are unsupervised multitask learners.
\newblock \emph{OpenAI blog}, 1\penalty0 (8):\penalty0 9, 2019.

\bibitem[Reddi et~al.(2020)Reddi, Charles, Zaheer, Garrett, Rush,
  Kone{\v{c}}n{\`y}, Kumar, and McMahan]{reddi2020adaptive}
Reddi, S., Charles, Z., Zaheer, M., Garrett, Z., Rush, K., Kone{\v{c}}n{\`y},
  J., Kumar, S., and McMahan, H.~B.
\newblock Adaptive federated optimization.
\newblock \emph{arXiv preprint arXiv:2003.00295}, 2020.

\bibitem[Sandler et~al.(2018)Sandler, Howard, Zhu, Zhmoginov, and
  Chen]{sandler2018mobilenetv2}
Sandler, M., Howard, A.~G., Zhu, M., Zhmoginov, A., and Chen, L.
\newblock {MobileNetV2: I}nverted residuals and linear bottlenecks.
\newblock In \emph{Proceedings of the {IEEE} Conference on Computer Vision and
  Pattern Recognition {(CVPR)}}, pp.\  4510--4520, 2018.

\bibitem[Sattler et~al.(2020{\natexlab{a}})Sattler, Marban, Rischke, and
  Samek]{sattler2020communication}
Sattler, F., Marban, A., Rischke, R., and Samek, W.
\newblock Communication-efficient federated distillation.
\newblock \emph{arXiv preprint arXiv:2012.00632}, 2020{\natexlab{a}}.

\bibitem[Sattler et~al.(2020{\natexlab{b}})Sattler, M{\"u}ller, and
  Samek]{sattler2020clustered}
Sattler, F., M{\"u}ller, K.-R., and Samek, W.
\newblock Clustered federated learning: Model-agnostic distributed multitask
  optimization under privacy constraints.
\newblock \emph{IEEE Trans. Neural Netw. Learn. Syst.}, pp.\  1--13,
  2020{\natexlab{b}}.

\bibitem[Sattler et~al.(2020{\natexlab{c}})Sattler, Wiedemann, M{\"u}ller, and
  Samek]{sattler2019robust}
Sattler, F., Wiedemann, S., M{\"u}ller, K.-R., and Samek, W.
\newblock Robust and communication-efficient federated learning from non-iid
  data.
\newblock \emph{IEEE Trans. Neural Netw. Learn. Syst.}, 31\penalty0
  (9):\penalty0 3400--3413, 2020{\natexlab{c}}.

\bibitem[Schapire(1999)]{schapire1999brief}
Schapire, R.~E.
\newblock A brief introduction to boosting.
\newblock In \emph{Proceedings of the 16th International Joint Conference on
  Artificial Intelligence {(IJCAI)}}, pp.\  1401--1406, 1999.

\bibitem[Seo et~al.(2020)Seo, Park, Oh, Bennis, and Kim]{seo2020fd}
Seo, H., Park, J., Oh, S., Bennis, M., and Kim, S.
\newblock Federated knowledge distillation.
\newblock \emph{arXiv preprint arXiv:2011.02367}, 2020.

\bibitem[Sharkey(1996)]{sharkey1996combining}
Sharkey, A. J.~C.
\newblock On combining artificial neural nets.
\newblock \emph{Connect. Sci.}, 8\penalty0 (3):\penalty0 299--314, 1996.

\bibitem[Sheller et~al.(2020)Sheller, Edwards, Reina, Martin, Pati, Kotrotsou,
  Milchenko, Xu, Marcus, Colen, et~al.]{sheller2020federated}
Sheller, M.~J., Edwards, B., Reina, G.~A., Martin, J., Pati, S., Kotrotsou, A.,
  Milchenko, M., Xu, W., Marcus, D., Colen, R.~R., et~al.
\newblock Federated learning in medicine: {F}acilitating multi-institutional
  collaborations without sharing patient data.
\newblock \emph{Scientific Reports}, 10\penalty0 (1):\penalty0 1--12, 2020.

\bibitem[Smith et~al.(2017)Smith, Chiang, Sanjabi, and
  Talwalkar]{smith2017fedMTL}
Smith, V., Chiang, C., Sanjabi, M., and Talwalkar, A.~S.
\newblock Federated multi-task learning.
\newblock In \emph{Advances in Neural Information Processing Systems
  {(NeurIPS)}}, volume~30, pp.\  4424--4434, 2017.

\bibitem[Sollich \& Krogh(1995)Sollich and Krogh]{sollich1995learning}
Sollich, P. and Krogh, A.
\newblock Learning with ensembles: {H}ow overfitting can be useful.
\newblock In \emph{Advances in Neural Information Processing Systems
  {(NeurIPS)}}, volume~8, pp.\  190--196, 1995.

\bibitem[Sun \& Lyu(2020)Sun and Lyu]{sun2020federated}
Sun, L. and Lyu, L.
\newblock Federated model distillation with noise-free differential privacy.
\newblock \emph{arXiv preprint arXiv:2009.05537}, 2020.

\bibitem[Wang et~al.(2019)Wang, Mathews, Kiddon, Eichner, Beaufays, and
  Ramage]{wang2019finetuning}
Wang, K., Mathews, R., Kiddon, C., Eichner, H., Beaufays, F., and Ramage, D.
\newblock Federated evaluation of on-device personalization.
\newblock \emph{arXiv preprint arXiv:1910.10252}, 2019.

\bibitem[Wang \& Isola(2020)Wang and Isola]{wang2020understanding}
Wang, T. and Isola, P.
\newblock Understanding contrastive representation learning through alignment
  and uniformity on the hypersphere.
\newblock In \emph{International Conference on Machine Learning}, pp.\
  9929--9939. PMLR, 2020.

\bibitem[Wu et~al.(2020)Wu, Chen, and Wang]{wu2020dpFedMTL}
Wu, H., Chen, C., and Wang, L.
\newblock A theoretical perspective on differentially private federated
  multi-task learning.
\newblock \emph{arXiv preprint arXiv:2011.07179}, 2020.

\bibitem[You et~al.(2017)You, Xu, Xu, and Tao]{you2017learning}
You, S., Xu, C., Xu, C., and Tao, D.
\newblock Learning from multiple teacher networks.
\newblock In \emph{Proceedings of the 23rd {ACM} {SIGKDD} International
  Conference on Knowledge Discovery and Data Mining {(KDD)}}, pp.\  1285--1294,
  2017.

\bibitem[Yuksel et~al.(2012)Yuksel, Wilson, and Gader]{yuksel2012twenty}
Yuksel, S.~E., Wilson, J.~N., and Gader, P.~D.
\newblock Twenty years of mixture of experts.
\newblock \emph{{IEEE} Trans. Neural Networks Learn. Syst.}, 23\penalty0
  (8):\penalty0 1177--1193, 2012.

\bibitem[Zhang et~al.(2020{\natexlab{a}})Zhang, Kuang, You, Shen, Xiao, Zhang,
  Wu, Zhuang, and Li]{zhang2020federated}
Zhang, F., Kuang, K., You, Z., Shen, T., Xiao, J., Zhang, Y., Wu, C., Zhuang,
  Y., and Li, X.
\newblock Federated unsupervised representation learning.
\newblock \emph{arXiv preprint arXiv:2010.08982}, 2020{\natexlab{a}}.

\bibitem[Zhang et~al.(2015)Zhang, Zhao, and LeCun]{Zhang2015CharacterlevelCN}
Zhang, X., Zhao, J.~J., and LeCun, Y.
\newblock Character-level convolutional networks for text classification.
\newblock In \emph{Advances in Neural Information Processing Systems
  {(NeurIPS)}}, volume~28, pp.\  649--657, 2015.

\bibitem[Zhang et~al.(2018)Zhang, Zhou, Lin, and Sun]{zhang2018shufflenet}
Zhang, X., Zhou, X., Lin, M., and Sun, J.
\newblock {ShuffleNet: A}n extremely efficient convolutional neural network for
  mobile devices.
\newblock In \emph{Proceedings of the IEEE Conference on Computer Vision and
  Pattern Recognition {(CVPR)}}, pp.\  6848--6856, 2018.

\bibitem[Zhang et~al.(2020{\natexlab{b}})Zhang, Yao, Yang, Yan, Gonzalez, and
  Mahoney]{zhang2020benchmarking}
Zhang, Z., Yao, Z., Yang, Y., Yan, Y., Gonzalez, J.~E., and Mahoney, M.~W.
\newblock Benchmarking semi-supervised federated learning.
\newblock \emph{arXiv preprint arXiv:2008.11364}, 2020{\natexlab{b}}.

\bibitem[Zhao et~al.(2018)Zhao, Li, Lai, Suda, Civin, and
  Chandra]{zhao2018federated}
Zhao, Y., Li, M., Lai, L., Suda, N., Civin, D., and Chandra, V.
\newblock Federated learning with non-iid data.
\newblock \emph{arXiv preprint arXiv:1806.00582}, 2018.

\bibitem[Zhou et~al.(2020)Zhou, Pu, Ma, Li, and Wu]{zhou2020distilled}
Zhou, Y., Pu, G., Ma, X., Li, X., and Wu, D.
\newblock Distilled one-shot federated learning.
\newblock \emph{arXiv preprint arXiv:2009.07999}, 2020.

\bibitem[Zhu et~al.(2015)Zhu, Kiros, Zemel, Salakhutdinov, Urtasun, Torralba,
  and Fidler]{Zhu_2015_ICCV}
Zhu, Y., Kiros, R., Zemel, R.~S., Salakhutdinov, R., Urtasun, R., Torralba, A.,
  and Fidler, S.
\newblock Aligning books and movies: {T}owards story-like visual explanations
  by watching movies and reading books.
\newblock In \emph{Proceedings of the 2015 {IEEE} International Conference on
  Computer Vision {(ICCV)}}, pp.\  19--27, 2015.

\end{thebibliography}
% \bibliographystyle{icml2021}

\clearpage

\setcounter{section}{0}
\renewcommand{\thesection}{\Alph{section}}

%\section{Supplement}
\twocolumn[
\icmltitle{\textsc{FedAUX}: Leveraging Unlabeled Auxiliary Data in Federated Learning\\- \textsc{Supplementary Materials} -}
%\icmltitle{\textsc{FedAUX}: Deriving utility from unlabeled Auxiliary data in Federated Distillation}

% It is OKAY to include author information, even for blind
% submissions: the style file will automatically remove it for you
% unless you've provided the [accepted] option to the icml2021
% package.

% List of affiliations: The first argument should be a (short)
% identifier you will use later to specify author affiliations
% Academic affiliations should list Department, University, City, Region, Country
% Industry affiliations should list Company, City, Region, Country

% You can specify symbols, otherwise they are numbered in order.
% Ideally, you should not use this facility. Affiliations will be numbered
% in order of appearance and this is the preferred way.
% \icmlsetsymbol{equal}{*}

% \begin{icmlauthorlist}
% \icmlauthor{Felix Sattler}{}
% \icmlauthor{Tim Korjakow}{}
% \icmlauthor{Roman Rischke}{}
% \icmlauthor{Wojciech Samek}{}

% \end{icmlauthorlist}

% \icmlaffiliation{to}{Department of Artificial Intelligence, Fraunhofer HHI, Berlin, Germany}

% \icmlcorrespondingauthor{Felix Sattler}{felix.sattler@hhi.fraunhofer.de}
% \icmlcorrespondingauthor{Wojciech Samek}{wojciech.samek@hhi.fraunhofer.de}

% % You may provide any keywords that you
% % find helpful for describing your paper; these are used to populate
% % the "keywords" metadata in the PDF but will not be shown in the document
% \icmlkeywords{Machine Learning, ICML}

\vskip 0.3in
]
\section{Extended Related Work Discussion}
\label{supp:related}

%\textsc{FedAVG} \macrocite{mcmahan2017communication}, local SGD \macrocite{lin2020localSGD} and their adaptations to handle heterogeneous environments require to communicate (differential) model updates in each round, which is challenging for large models especially when having limited communication bandwidth and/or high communication cost \macrocite{sattler2019robust}. Different papers have investigated communication-efficient federated learning, see e.g. \macrocite{mcmahan2017communication, sattler2019robust} and the references therein. The introduction of model distillation techniques into federated learning \macrocite{jeong2018distill, lin2020ensemble, itahara2020distill} opens new opportunities both with respect to communication-efficiency and the challenge of a heterogeneous environment.

\textbf{Ensemble Distillation in Federated Learning:}

A new family of Federated Learning methods leverages model distillation \suppcite{hinton2015distill} to aggregate the client knowledge \suppcite{jeong2018distill, lin2020ensemble, itahara2020distill, chen2020feddistill}. These Federated Distillation (FD) techniques have at least three distinct advantages over prior, parameter averaging based methods and related work can be organized according to which of these aspects it primarily focuses on. 

First, Federated Distillation enables aggregation of client knowledge independent of the model architecture and thus allows clients to train models of different structure, which gives additional flexibility, especially in hardware-constrained settings. \textsc{FedMD} \suppcite{li2019fedmd},  Cronus \suppcite{chang2019cronus} and \textsc{FedH2L} \suppcite{li2021fedh2l} address this aspect. FedMD additionally requires to locally pre-train on the \emph{labeled} public data which makes it difficult to perform a fair numerical comparison. FedH2L requires communication of soft-label information after every gradient descent step and is thus not suitable for most practical FL applications where communication channels are intermittent. Cronus addresses aspects of robustness to adversaries but is shown to perform consistently worse than \textsc{FedAVG} in conventional FL. While we do not focus on this aspect, our proposed approach is flexible enough to handle heterogeneous client models (c.f. Appendix \ref{supp:algorithm}).

Second, Federated Distillation has advantageous communication properties. As models are aggregated by means of distillation instead of parameter averaging it is no longer necessary to communicate the raw parameters. Instead it is sufficient for the clients to only send their soft-label predictions on the distillation data. Consequently, the communication in FD scales with the size of the distillation data set and not with the size of the jointly trained model as in the classical parameter averaging based FL. This leads to communication savings, especially if the local models are large and the distillation data set is small. 
Jeong~et.~al~ and subsequent work \suppcite{jeong2018distill, itahara2020distill, seo2020fd, sattler2020communication} focus on this aspect. These methods however are computationally more expensive for the resource constrained clients, as distillation needs to be performed locally and perform worse than parameter averaging based training after the same number of communication rounds. Our proposed approach relies on communication of full models and thus requires communication at the order of conventional parameter averaging based methods.

Third, when combined with parameter averaging, Federated Distillation methods achieve better performance than purely parameter averaging based techniques.  Both the authors in \suppcite{lin2020ensemble} and \suppcite{chen2020feddistill} propose FL protocols, which are based on classical \textsc{FedAVG} and perform ensemble distillation after averaging the received client updates at the server to improve performance. \textsc{FedBE}, proposed by \suppcite{chen2020feddistill}, additionally combines client predictions by means of a Bayesian model ensemble to further improve robustness of the aggregation. Our work primarily focuses on this latter aspect. Building upon the work of \suppcite{lin2020ensemble}, we additionally leverage the auxiliary distillation data for unsupervised pre-training and weigh the client predictions in the distillation step according to their certainty scores to better cope with settings where the client’s data generating distributions are statistically heterogeneous. 
 
We also mention the related work by Guha~et~al.~\suppcite{guha2019one-shot}, which proposes a one-shot distillation method for convex models, where the server distills the locally optimized client models in a single round as well as the work of \suppcite{sun2020federated} which addresses privacy issues in Federated Distillation. Federated one-shot distillation is also addressed in \suppcite{zhou2020distilled}. Federated Distillation for edge-learning was proposed in \suppcite{ahn2019wireless}.

\textbf{Weighted Ensembles:}
The study of weighted ensembles started around the '90s with the work by \suppcite{hashem1993ensemble, perrone1993ensemble, sollich1995learning}. 
A weighted ensemble of models combines the output of the individual models by means of a weighted average in order to improve the overall  generalization performance. The weights allow to indicate the percentage of trust or expected performance for each individual model.
See \suppcite{sharkey1996combining, opitz1999popular} for an overview of ensemble methods. 
Instead of giving each client a static weight in the aggregation step of distillation, we weight the clients on an instance base as in \suppcite{jimenez1998dynamically}, i.e., each clients prediction is weighted using a data-dependent certainty score. Weighted combinations of weak classifiers are also commonly leveraged in centralized settings in the context of of mixture of experts and boosting methods \suppcite{yuksel2012twenty, masoudnia2014mixture, schapire1999brief}.

\textbf{Data Heterogeneity in Federated Learning:}
As the training data is generated independently on the participation devices, Federated Learning problems are typically characterised by statistically heterogeneous client data \suppcite{mcmahan2017communication}. It is well known, that conventional FL algorithms like \textsc{FedAVG} \suppcite{mcmahan2017communication} perform best on statistically homogeneous data and suffer severely in this (“non-iid”) setting  \suppcite{zhao2018federated, li2020covergence}. A number of different studies \suppcite{li2019fedprox, zhao2018federated, sattler2019robust, chen2020feddistill} have tried to address this issue, but relevant performance improvements so far have only been possible under strong assumptions. For instance \suppcite{zhao2018federated} assume that the server has access to \emph{labeled} public data from the \emph{same} distribution as the clients. In contrast, we only assume that the server has access to \emph{unlabeled} public data from a potentially \emph{deviating} distribution. Other approaches \suppcite{sattler2019robust} require high-frequent communication, with up to thousands of communication rounds, between server and clients, which might be prohibitive in a majority of FL applications where communication channels are intermittent and slow. In contrast, our proposed approach can drastically improve FL performance on non-iid data even after just one single communication round.
For completeness, we note that there exists also a different line of research, which aims to address data heterogeneity in FL via meta- and multi-task learning. Here, separate models are trained for each client \suppcite{smith2017fedMTL, wu2020dpFedMTL} or clients are grouped into different clusters with similar distributions \suppcite{ghosh2019noniid, sattler2020clustered}.

\textbf{Unlabeled Data in Federated Learning:}
To the best of our knowledge, there do not exist any prior studies on the use of unlabeled auxiliary data in FL outside of Federated Distillation methods.
Federated semi-supervised learning techniques \suppcite{zhang2020benchmarking, jeong2020federated} assume that clients hold both labeled and unlabeled private data from the local training distribution. In contrast, we assume that the server has access to public unlabeled data that may differ in distribution from the local client data. Federated self-supervised representation learning \suppcite{zhang2020federated} aims to train a feature extractor on private unlabeled client data. In contrast, we leverage self-supervised representation learning at the server to find a suitable model initialization.

\textbf{Personalization and Federated Transfer Learning:}
The aim of Transfer Learning is to transfer learned knowledge from a specific domain or task to related domains or tasks. 
Transfer learning methods are of particular interest in FL settings where the client's local data generating distributions are statistically heterogeneous.  
To address the statistical heterogeneity, methods for personalizing the server model to the client's local distributions, e.g. by using distillation \suppcite{li2019fedmd}, parameter fine-tuning \suppcite{wang2019finetuning, mansour2020personalization} or regularization \suppcite{li2019fedprox}, have been proposed. Transferring knowledge from one domain to another domain raises the question of the generalization capabilities and domain adaptation theory gives answers in the form of generalization bounds. Particularly, multiple-source domain adaptation theory \suppcite{mansour2008domain, bendavid2010domain, hoffman2018domain}, which considers the capabilities of transferring knowledge from multiple source domains to some target domain, is relevant for FL.
One interesting question when having knowledge in multiple source domains is how to weight each individual source domain in the process of transferring knowledge to the target domain.
In the \textsc{FedDF} algorithm \suppcite{lin2020ensemble}, the client's local hypotheses are uniformly averaged to obtain a global hypothesis and it is remarked that domain adaptation theory \suppcite{mansour2008domain, hoffman2018domain} has shown such standard convex combinations of source hypotheses not to be robust for the target domain.   
A distribution-weighted combination of the local hypotheses, as suggested by domain adaptation theory \suppcite{mansour2008domain} \suppcite{hoffman2018domain}, based on a privacy-preserving local distribution estimation is posed as an open problem for FL in \suppcite{lin2020ensemble}. We address exactly this open question. 

\begin{algorithm}[t!]
   \caption{\textsc{FedAUX} Preparation Phase (with different model prototypes $\mathcal{P}$)}
   \label{alg:FedAUXprep}
\begin{algorithmic}
\STATE \textbf{init:} Split $D^- \cup D_{distill} \leftarrow D_{aux}$
\STATE \textbf{init:} HashMap $\mathcal{R}$ that maps client $i$ to model prototype $P$
\STATE \underline{Server does:}
\FOR{each model prototype $P\in\mathcal{P}$}
\STATE $h^P_0\leftarrow \text{train\_self\_supervised}(h^P, D_{aux})$
\ENDFOR
\FOR{each client $i \in \{1,..,n\}$ \textbf{in parallel}}
\STATE \underline{Client $i$ does:}
\STATE $P\leftarrow \mathcal{R}[i]$
\STATE $\sigma^2\leftarrow \frac{8\ln(1.25\delta^{-1})}{\varepsilon^2\lambda^2(|D_i|+|D^-|)^2}$
\STATE $w_i^* \leftarrow \arg\min_{w} J(w, h^P_0, D_i, D^-) + \mathcal{N}(\mathbf{0}, I\sigma^2)$
\STATE $\gamma_i\leftarrow\max_{x\in D_i\cup D^-}\|h^P_0(x)\|$
%\STATE with $s_i : x\mapsto (1+\exp(-\langle w_i^*, \gamma^{-1}h^{P}_0(x)\rangle))^{-1}+\xi$
\ENDFOR 
\STATE \underline{Server does:}
\FOR{$i=1,..,n$}
\STATE create HashMap 
\STATE $s_i \leftarrow \{x\mapsto(1+\exp(-\langle w_i^*, \gamma_i^{-1}h^{P}_0(x)\rangle))^{-1}+\xi$ for $x\in D_{distill}\}$
\ENDFOR
\end{algorithmic}
\end{algorithm}

\section{Data Splitting Methodology}
\label{supp:data_splitting}
We split the training data among the clients using the common Dirichlet splitting strategy proposed in \suppcite{hsu2019measuring} and later used in \suppcite{lin2020ensemble} and \suppcite{chen2020feddistill}. This approach allows us to smoothly adapt the level of heterogeneity in the client data via the concentration parameter $\alpha$. To generate the data split, we sample $c$ vectors
\begin{align}
    p_1,..,p_c\sim\text{Dir}(\alpha),
\end{align}
where c is the number of classes, from the symmetric $n$-categorical Dirichlet distribution. For all $p_i\in\mathbb{R}_{\geq 0}^n$ it then holds $\|p_i\|_1=1$. The vectors are then stacked To address the statistical heterogeneity, methods for personalizing the server model to the client's local distributions, e.g. by using distillation \suppcite{li2019fedmd}, parameter fine-tuning \suppcite{wang2019finetuning, mansour2020personalization} or regularization \suppcite{li2019fedprox}, have been proposed. Transferring knowledge from one domain to another domain raises the question of the general

into a matrix 
\begin{align}
    P = [p_1, .., p_c]\in\mathbb{R}^{n,c}
\end{align}
which is standardized, by repeatedly normalizing the columns and rows. This process converges quickly and is stopped after 1000 iterations. Let $M_j$ be the amount of data points belonging to class $j$ in the training data set. Each client $i$ is then assigned $P_{i,j}M_j$ (non-overlapping) data points from all classes $j=1,..,c$. Figure \ref{fig:splitting_method} illustrates the splitting procedure and displays random splits of data for $n=20$ and $c=10$. In all our experiments, the data splitting process is controlled by a random seed, to ensure that the different baseline methods are all trained on the same split of data. 

\begin{figure*}[t!]
    \centering
    \includegraphics[width=1.0\textwidth]{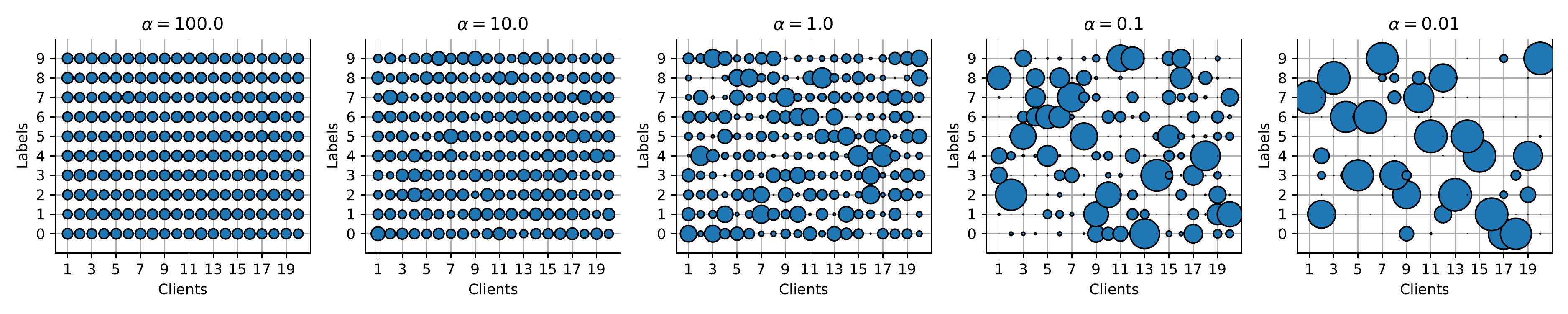}
    \caption{Illustration of the Dirichlet data splitting strategy used throughout the paper. Dot size represents number of data points each client holds from any particular class. Lower values of $\alpha$ lead to more heterogeneous splits of data.}
    \label{fig:splitting_method}
    
    % \includegraphics[width=1.0\textwidth]{images_supplement/alpha_cifar100.pdf}
    % \caption{Illustration of the Dirichlet data splitting strategy used throughout the paper. Dot size represents number of data points each client holds from any particular class. Lower values of $\alpha$ lead to more heterogeneous splits of data.}
    % \label{fig:splitting_method}
\end{figure*}

\begin{algorithm}[t!]
   \caption{\textsc{FedAUX} Training Phase (with different model prototypes $\mathcal{P}$). Training requires feature extractors $h_0^P$ and scores $s_i$ from Alg. \ref{alg:FedAUXprep}. The same $D^- \cup D_{distill} \leftarrow D_{aux}$ as in Alg. \ref{alg:FedAUXprep} is used. Choose learning rate $\eta$ and set $\xi=10^{-8}$.}
   \label{alg:FedAUXtrain}
\begin{algorithmic}
\STATE \textbf{init:} HashMap $\mathcal{R}$ that maps client $i$ to model prototype $P$
\STATE \textbf{init:} Inverse HashMap $\tilde{\mathcal{R}}$ that maps model prototype $P$ to set of clients (s.t. $i\in\tilde{\mathcal{R}}[\mathcal{R}[i]]~\forall i$) 
\STATE \textbf{init:} Initialize model prototype weights $\theta^P$ with feature extractor weights $h^P$ from Alg. \ref{alg:FedAUXprep}
\FOR{communication round $t=1,..,T$}
\STATE select subset of clients $\mathcal{S}_t\subseteq \{1,..,n\}$
\FOR{selected clients $i \in \mathcal{S}_t$ \textbf{in parallel}}
\STATE \underline{Client $i$ does:}
\STATE $\theta_i\leftarrow \text{train}(\theta_0\leftarrow\theta^{\mathcal{R}[i]}, D_i)$\hfill \# Local Training
\ENDFOR
\STATE \underline{Server does:}
\FOR{each model prototype $P\in\mathcal{P}$}
\STATE $\theta^P\leftarrow \sum_{i\in \mathcal{S}_t\cap \tilde{\mathcal{R}}[P]}\frac{|D_i|}{\sum_{l\in \mathcal{S}_t\cap \tilde{\mathcal{R}}[P]}|D_l|} \theta_i$ \hfill \# Parameter
\STATE \hfill \# Averaging
\FOR{mini-batch $x \in D_{distill}$}
\STATE $\tilde{y} \leftarrow \sigma\left(\frac{\sum_{i\in \mathcal{S}_t} s_i[x]f_i(x, \theta_i)}{\sum_{i\in \mathcal{S}_t} s_i[x]}\right)$\hfill \# Can be arbitrary
\STATE $\theta^P \leftarrow \theta^P - \eta\frac{\partial D_{KL}(\tilde{y}, \sigma(f(x, \theta^P)))}{\partial \theta^P}$ \hfill \# Optimizer
\ENDFOR
\ENDFOR
\ENDFOR
\end{algorithmic}
\end{algorithm}

\begin{figure*}[t!]
    \centering
    \includegraphics[width=0.95\textwidth]{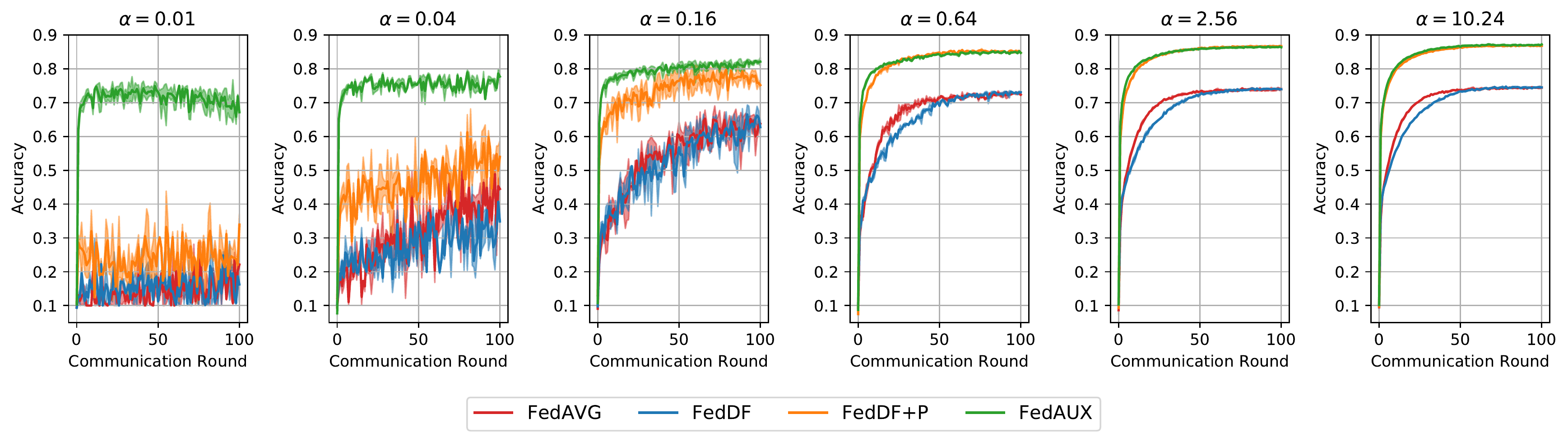}
    \caption{Detailed training curves for ResNet-8 trained on CIFAR-10, $n=80$ Clients, $C=40\%$.}
    \label{fig:training_curves_resnet}

    \includegraphics[width=0.95\textwidth]{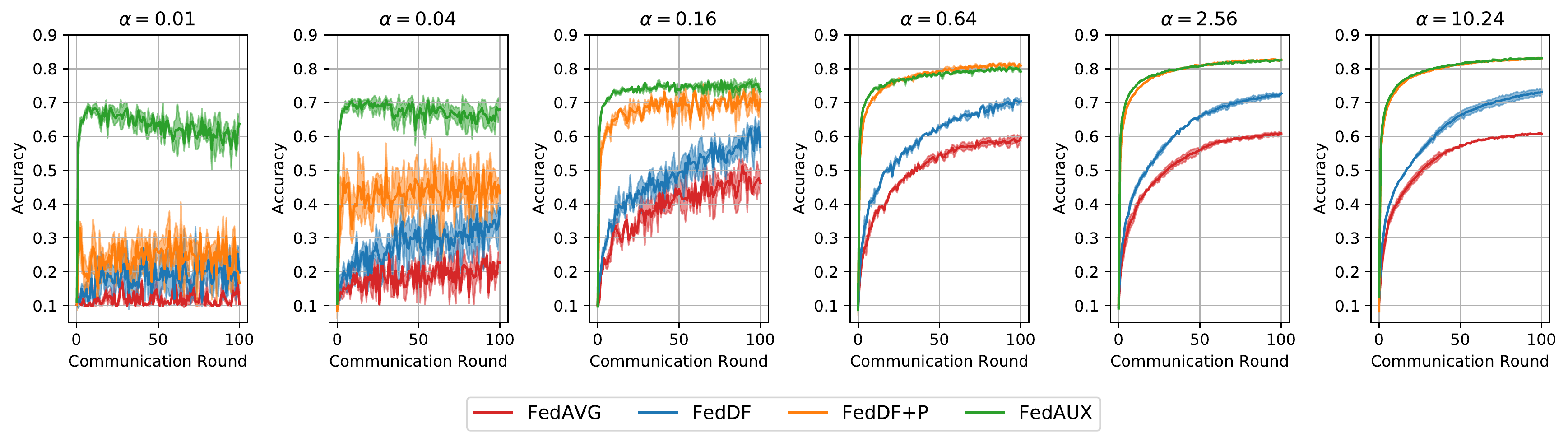}
    \caption{Detailed training curves for MobileNetv2 trained on CIFAR-10, $n=100$ Clients, $C=40\%$.}
    \label{fig:training_curves_mobilenet}

    \includegraphics[width=0.95\textwidth]{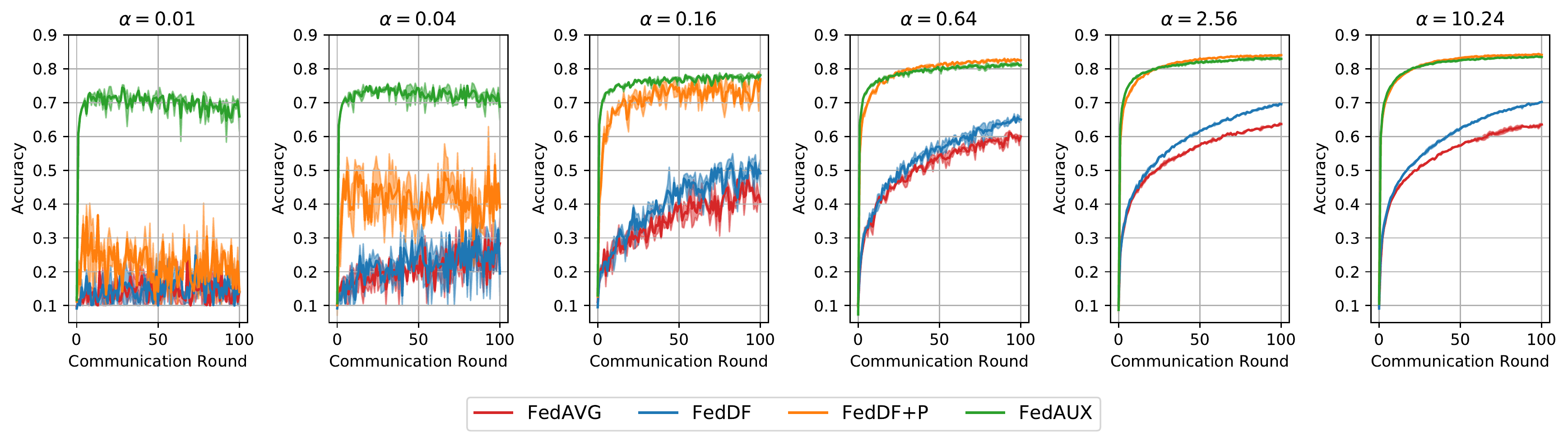}
    \caption{ Shufflenet trained on CIFAR-10, $n=100$ Clients, $C=40\%$.}
    \label{fig:training_curves_shufflenet}

    \includegraphics[width=0.95\textwidth]{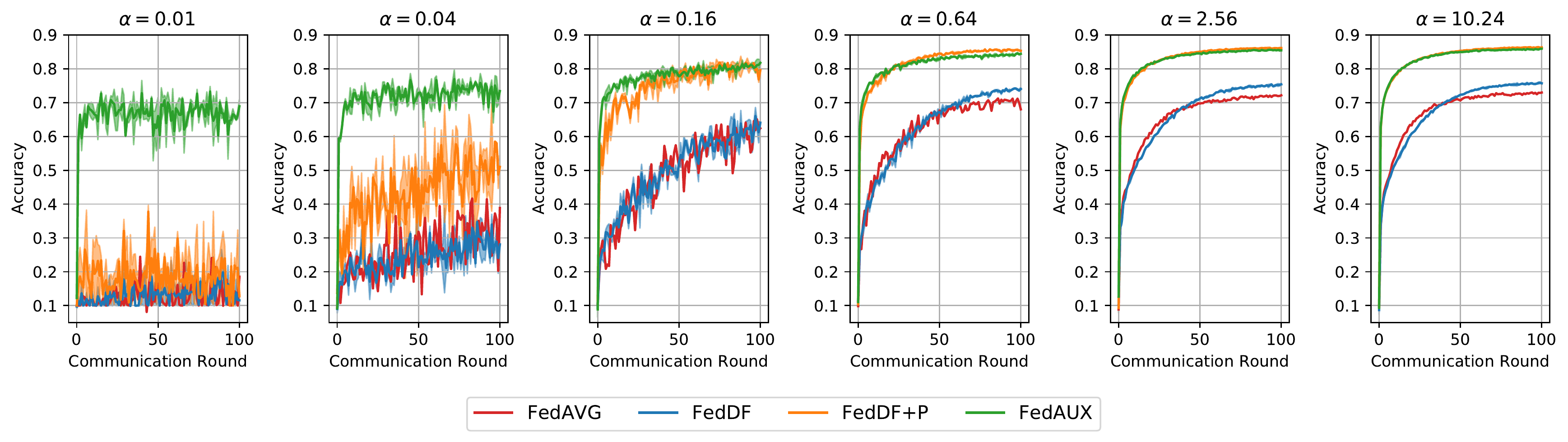}
    \caption{Detailed training curves for mixed models trained on CIFAR-10. 20 each train ResNet8, MobileNetv2 and Shufflenet respectively.}
    \label{fig:training_curves_mixed}
\end{figure*}

\section{Detailed Algorithm}
\label{supp:algorithm}
The training procedure of \textsc{FedAUX} can be divided into a preparation phase, which is given in Alg. \ref{alg:FedAUXprep} and  a training phase, which is given in Alg. \ref{alg:FedAUXtrain}. We describe the general setting where clients may hold different model prototypes $P$ from a set of prototypes $\mathcal{P}$. This general setting simplifies to the setting described in Sec.~\ref{sec:method} if $|\mathcal{P}|=1$.

\textbf{Preparation Phase:} In the preparation phase, the server uses the unlabeled auxiliary data $D_{aux}$, to pre-train the feature extractor $h^P$ for each model prototype $P$ using self-supervised training. Suitable methods for self-supervised pre-training are contrastive representation learning \suppcite{chen2020simple}, or self-supervised language modeling/ next-token prediction \suppcite{devlin2018bert}. The pre-trained feature extractors $h^P_0$ are then communicated to the clients and used to initialize part of the local classifier $f=g\circ h$. The server also communicates the negative data $D^-$ to the clients (in practice we can instead communicate the extracted features $\{|h_0^P(x)|x\in D^-\}$ of the raw data $D^-$ to save communication). Each client then optimizes the logistic similarity objective $J$ \eqref{eq:ERM} and sanitizes the output by adding properly scaled Gaussian noise. Finally, the sanitized scoring model $w_i^*$ is communicated to the server, where it is used to compute certainty scores $s_i$ on the distillation data (the certainty scores can also be computed on the clients, however this results in additional communication of distillation data and scores). 

\textbf{Training Phase:} The training phase is carried out in $T$ communication rounds. In every round $t\leq T$, the server randomly selects a subset $\mathcal{S}_t$ of the overall client population and transmits to them the latest server models $\theta^\mathcal{R}[i]$, which match their model prototype $P$  (in round $t=1$ only the pre-trained feature extractor $h^P_0$ is transmitted). Each selected client updates it's local model by performing multiple steps of stochastic gradient descent (or it's variants) on it's local training data. This results in an updated parameterization $\theta_i$ on every client, which is communicated to the server. After all clients have finished their local training, the server gathers the updated parameters $\theta_i$. For each model prototype $P$ the corresponding parameters are then aggregated by weighted averaging. Using the model averages as a starting point, for each prototype the server then distills a new model, based on the client's certainty-weighted predictions.

\section{Qualitative Comparison with Baseline Methods}
\label{supp:qualitative}
Table \ref{tab:compare_qualitative} gives a qualitative comparison between \textsc{FedAUX} and the baseline methods \textsc{FedAVG} and \textsc{FedDF}. 
\begin{itemize}
    \item Compared with \textsc{FedAVG} and \textsc{FedDF}, \textsc{FedAUX} additionally requires the clients to once solve the $\lambda$-strongly convex ERM \eqref{eq:ERM}. For this problem linearly convergent algorithms are known \suppcite{liu1989limited} and thus the computational overhead is negligible compared with the complexity of multiple rounds of locally training deep neural networks.
    \item  \textsc{FedAUX} also adds computational load to the server for self-supervised pre-training and computation of the certainty scores $s_i$. As the server is typically assumed to have massively stronger computational resources than the clients, this can be neglected.
    \item  Once, in the preparation phase of \textsc{FedAUX}, the scoring models $w_i^*$ need to be communicated from the clients to the server. The overhead of communicating these $H$-dimensional vectors, where $H$ is the feature dimension, is negligible compared to the communication of the full models $f_i$.
    \item  \textsc{FedAUX} also requires the communication of the negative data $D^-$ and the feature extractor $h_0$ from the server to the clients. The overhead of sending $h_0$ is lower than sending the full model $f$, and thus the total downstream communication is increased by less than a factor of $(T+1)/T$. The overhead of sending $D^-$ is small (in our experiments $|D^-|=0.2|D_{aux}|$) and can be further reduced by sending extracted features $\{|h_0^P(x)|x\in D^-\}$ instead of the full data. For instance, in our experiments with ResNet-8 and CIFAR-100  we have $|D^-|=12000$ and $h_0^P(x)\in\mathbb{R}^{512}$, resulting in a total communication overhead of $12000\times512\times4B=24.58$MB for $D^-$. For comparison the total communication overhead of once sending the parameters of ResNet-8 (needs to be done $T$ times) is $19.79$MB.
    \item Communicating the scoring models $w_i^*$ incurs additional privacy loss for the clients. Using our proposed sanitation mechanism this process is made $(\varepsilon, \delta)$-differentially private. Our experiments in section \ref{sec:ex_privacy} demonstrate that \textsc{FedAUX} can achieve drastic performance improvements, even under conservative privacy constraints. All empirical results reported are obtained with $(\varepsilon,\delta)$ differential privacy at $\varepsilon=0.1$ and $\delta=10^{-5}$.
    \item Finally, \textsc{FedAUX} makes the additional assumption that unlabeled auxiliary data is available to the server. This assumption is made by all Federated Distillation methods including \textsc{FedDF}.
\end{itemize}

\definecolor{Gray}{gray}{0.9}
\begin{table*}[t!]
    \centering
    \caption{\textbf{Qualitative Comparison:} Complexity, communication overhead, privacy loss after $T$ communication rounds as well as implicit assumptions made by different Federated Learning methods.}
    \label{tab:compare_qualitative}
    {\renewcommand{\arraystretch}{1.5}
    \begin{tabular}{p{2.5cm}|p{2.5cm}p{2.5cm}p{3.9cm}p{3.6cm}}
    \toprule
    & \textsc{FedAVG} & \textsc{FedDF} & \textsc{FedAUX} (preparation phase) & \textsc{FedAUX} (training phase)\\
    \midrule
    \rowcolor{Gray}
    Operations (Clients)     &  Local Training ($\times T$)& Local Training ($\times T$)& Solve $\lambda$-strongly convex ERM \eqref{eq:ERM} & Local Training ($\times T$) \\
    Operations (Server)    &  Model Averaging ($\times T$)& Model Averaging, Distillation ($\times T$) & Self-Supervised Pre-training of $h_0$, Computation of certainty scores $s_i$  & Model Averaging, Distillation ($\times T$) \\
    \rowcolor{Gray}
    Communication Clients $\rightarrow$ Server    & Model Parameters $f_i$ ($\times T$) & Model Parameters $f_i$ ($\times T$)& Scoring Models $w_i^*$ & Model Parameters $f_i$ ($\times T$)  \\
    Communication Server $\rightarrow$ Clients    & Model Parameters $f$ ($\times T$) & Model Parameters $f$ ($\times T$) & Negative Data $D^-$, Feature Extractor $h_0$ & Model Parameters $f$ ($\times T$) \\
    \rowcolor{Gray}
    Privacy Loss & Privacy loss of communicating $f_i$ ($\times T$) & Privacy loss of communicating $f_i$ ($\times T$)& $(\varepsilon, \delta)$-DP & Privacy loss of communicating $f_i$ ($\times T$)\\
    Assumptions & No Assumptions & Auxiliary Data & Auxiliary Data & Auxiliary Data  
    \end{tabular}
    }
    
\end{table*}

\section{Additional Results and Detailed Training Curves}
\label{supp:training_curves}
In this sections we give detailed training curves for the results shown in Figure \ref{fig:summary_distillation}. As can be seen, in the highly non-iid setting at $\alpha\in\{0.01,0.04\}$, all methods exhibit convergence issues. This behavior is well known in FL and is described for instance in \suppcite{zhao2018federated, sattler2019robust}. Notably, the performance of \textsc{FedAUX} after one single communication round exceeds the maximum achieved performance of all other methods over the entire course of training. At higher values of $\alpha\geq 0.16$ all methods train smoothly and validation performance asymptotically increases over the curse of training. \textsc{FedAUX} dominates all baseline methods at all communication rounds in the heterogeneous settings. In the mostly iid-setting at $\alpha=10.24$ \textsc{FedAUX} is en par with the pre-trained version of \textsc{FedDF}. 

Table \ref{tab:cifar100} compares performance of \textsc{FedAUX} to baseline methods on the CIFAR-100 data set. Again \textsc{FedAUX} outperforms \textsc{FedAVG} and \textsc{FedDF} across all level of data heterogeneity $\alpha$ and shows superior performance to the improved \textsc{FedDF+P} when data is highly heterogeneous at $\alpha=\{0.01, 0.04\}$. Interestingly in this setting \textsc{FedDF+P} manages to slightly outperform \textsc{FedAUX} at medium data heterogeneity levels $\alpha=\{0.16, 0.64\}$. This indicates that our proposed differentially private certainty scoring method may insufficiently approximate the true client certainty in this setting. We leave potential improvements of this mechanism for future work.

% Table \ref{tab:local_epochs} shows the effects of different local optimizers and numbers of training epochs on the performance of different FL methods. 

\begin{table}[t!]
    \centering
        \caption{Results on data sets with \textbf{higher number of classes.} Training ResNet-8 on \textbf{CIFAR-100}. Accuracy achieved after $T=100$ communication rounds by different Federated Distillation methods at different levels of data heterogeneity $\alpha$. STL-10 is used as auxiliary data set.}
        \label{tab:cifar100}
\begin{tabular}{lrrrrrr}
\toprule
& \multicolumn{6}{c}{$\alpha$}\\
\cline{2-7}
{} &  $0.01$ &  $0.04$ &  $0.16$ &  $0.64$ &  $2.56$ &  $10.24$ \\
\midrule
FedAVG  &    24.1 &    36.3 &    47.2 &    50.7 &    52.2 &     52.2 \\
FedDF   &    11.4 &    24.4 &    45.0 &    49.5 &    52.5 &     51.2 \\
FedDF+P &    18.2 &    42.0 &    \textbf{58.0} &    \textbf{60.8} &    61.6 &     62.0 \\
FedAUX  &    \textbf{34.1} &    \textbf{47.4} &    56.4 &    60.7 &    \textbf{62.5} &     \textbf{62.5} \\
\bottomrule
\end{tabular}
\end{table}

% \begin{table}[]
%     \centering
%         \caption{Evaluating the effects of different \textbf{local optimizers} and numbers of \textbf{local training} epochs $E$. Federated training of ShuffleNet with $n=20$ clients at a participation rate of $40\%$ in a homogeneous setting with $\alpha=10.24$. Maximum accuracy after $T=100$ rounds.}
%     \label{tab:local_epochs}
% \begin{tabular}{lrrrcrrr}
% \toprule
% & \multicolumn{3}{c}{SGD, $\text{lr}=0.1$} & & \multicolumn{3}{c}{Adam, $\text{lr}=0.001$}\\ 
% \cline{2-4}\cline{6-8}{}
% Epochs $E$ &     1  &     5  &     40 &    & 1  &     5  &     40 \\
% \midrule
% \textsc{FedAVG}   &  78.2 &  80.7 &  78.9 &&  76.0 &  78.7 &  82.0 \\
% \textsc{FedDF}    &  76.8 &  81.1 &  80.5 & & 79.0 &  79.8 &  82.3 \\
% \textsc{FedAVG+P} &  \textbf{85.2} &  \textbf{86.0} &  86.6 & & 87.3 &  88.2 &  88.3 \\
% \textsc{FedDF+P}  &  83.6 &  85.5 &  \textbf{86.7} & & \textbf{87.5} &  \textbf{88.3} &  88.6 \\
% \textsc{FedAUX}   &  83.8 &  85.6 &  86.7 & & 87.4 &  87.8 &  \textbf{88.7} \\
% \bottomrule
% \end{tabular}
% \end{table}

\section{Details on generating Imagenet subsets}
\label{supp:iamgenet_subsets}
To simulate the effects of a wide variety of auxiliary data sets on the training performance of \textsc{FedAUX}, we generate different structured subsets of the ImageNet data base (resized to $32\times 32\times 3$). Each subset is defined via a top-level Wordnet ID which is shown in Table \ref{tab:wordnetids}. To obtain the images from the subset, we select all leaf-node IDs of the respective top-level IDs via the Imagenet API
\begin{center}
\url{http://www.image-net.org/api/text/wordnet.structure.hyponym?wnid=<top-level ID>&full=1}     
\end{center}
and then take only those classes from the full Imagenet data set, which match these leaf-node IDs. Table \ref{tab:wordnetids} also shows the number of samples contained in every subset that was generated this way.

\begin{table}[t!]
    \centering
        \caption{\textbf{Auxiliary data sets} used in this study and their defining Wordnet IDs and data sets sizes.}
    \label{tab:wordnetids}
    \begin{tabular}{lll}
    \toprule
    Data set & Wordnet ID & Dataset Size\\
    \midrule
       Imagenet Devices  &  n03183080 & 165747\\
       Imagenet Birds & n01503061 & 76541\\
       Imagenet Animals & n00015388 & 510530\\
       Imagenet  Dogs & n02084071 & 147873\\
       Imagenet Invertebrates & n01905661 & 79300\\
       Imagenet Structures & n04341686 & 74400\\
       \bottomrule
    \end{tabular}
\end{table}

%Subset: fungus, # Samples: 7800
%Subset: plant_flora_plantlife, # Samples: 2600
%Subset: dogs, # Samples: 147873
%Subset: person, # Samples: 3900
%Subset: animal, # Samples: 510530
%Subset: artifact, # Samples: 667224
%Subset: devices, # Samples: 165747
%Subset: structure_construction, # Samples: 74400
%Subset: birds, # Samples: 76541

\section{Details on the Implementation and Results of the NLP Benchmarks}
\label{supp:transformer_implementation}

As mentioned in section 4.3 \textit{Evaluating} \textsc{FedAUX} \textit{on NLP Benchmarks} we used TinyBERT as a model for our NLP experiments. TinyBERT was pre-trained on Bookcorpus\footnote{\url{https://huggingface.co/datasets/bookcorpus}} which led us to select the same dataset as a public dataset in order to follow the methodology outlined in section \ref{sec:pretrain}. As private datasets we chose the AG News dataset\footnote{\url{https://huggingface.co/datasets/ag_news}} \suppcite{Zhang2015CharacterlevelCN}, a topic classification dataset, and the english texts from the Multilingual Amazon Reviews Corpus\footnote{\url{https://huggingface.co/datasets/amazon_reviews_multi}} \suppcite{marc_reviews}, which we use for predicting how many stars a review gets. The pre-trained weights and the tokenizer for TinyBERT are available at the corresponding repository\footnote{\url{https://huggingface.co/huawei-noah/TinyBERT_General_4L_312D}}. All experiments were conducted using $\epsilon = 0.1$ and $\delta=10^{-5}$ as differential privacy parameters, 1 epoch for local training and distillation, ten clients and 100\% participation rate as well as 160000 disjoint data points, which were sampled from BookCorpus, for the public and distillation datasets respectively. Furthermore the ADAM optimizer with a learning rate of $10^{-5}$ was used for both local training and distillation. The regularization strength of the logistic regression classifier was set to $0.01$. The batch size for $D_i, D^{-}$ and $D_{distill}$ was 32. Detailed results for figure \ref{fig:transformer} are depicted in table \ref{tab:transformer_results}.

\begin{table*}[t!]
    \centering
    \caption{\textbf{NLP Benchmarks} of different FL methods. Maximum accuracy achieved after $T=20$ communication rounds at participation-rate $C=100\%$.}
    \label{tab:transformer_results}
\begin{tabular}{lllcll}
\toprule
 & \multicolumn{2}{c}{AG News} & & \multicolumn{2}{c}{Amazon} \\
\cline{2-3}
\cline{5-6}
Method &            
$\alpha=0.01$ &            
$\alpha=1.0$ &
&
$\alpha=0.01$ &           
$\alpha=1.0$ \\
\midrule
\textsc{FedAVG+P}&
78.80$\pm$4.40 & 
\textbf{92.17$\pm$1.98} &
&
41.70$\pm$0.58 &          
\textbf{55.17$\pm$0.40} \\
\textsc{FedDF+P} &
78.05$\pm$7.64 &           
90.83$\pm$0.25 &
&
38.04$\pm$0.84 &          
54.63$\pm$0.66 \\
\textsc{FedAUX}& 
\textbf{85.04$\pm$1.21} &          
91.00$\pm$0.30 &
&
\textbf{49.11$\pm$0.22} &          
54.86$\pm$0.61 \\
\bottomrule
\end{tabular}
\end{table*}

\section{Hyperparameter Evaluation}
\label{supp:hyperparameter}
In this section we provide a detailed hyperparameter analysis for our proposed method and the baseline methods used in this study. For all methods we use the very popular Adam optimizer for both local training and distillation. We vary the learning rate in $\{1e-2, 1e-3, 1e-4, 1e-5\}$ for local training an distillation. For FedPROX, we vary the parameter $\lambda_{prox}$, controlling the proximal term in the training objective in $\{1e-2, 1e-3, 1e-4, 1e-5\}$. Figure \ref{fig:HPO} compares the maximum achieved accuracy after 50 communication rounds for the different methods and hyperparameter settings, for a FL setting with 20 clients training ResNet-8 on CIFAR-10 at a participation-rate of 40\%. The auxiliary data set we use is STL-10. 

For each method and each level of data heterogeneity, table \ref{tab:HPO} shows the accuracy of the best performing combination of hyperparameters. As we can see \textsc{FedAUX} matches the performance of the best performing methods in the iid setting with $\alpha=100.0$ and outperforms all other methods distinctively in the non-iid setting with $\alpha=0.01$.
\begin{figure*}[t!]
    \centering
    \includegraphics[width=\textwidth]{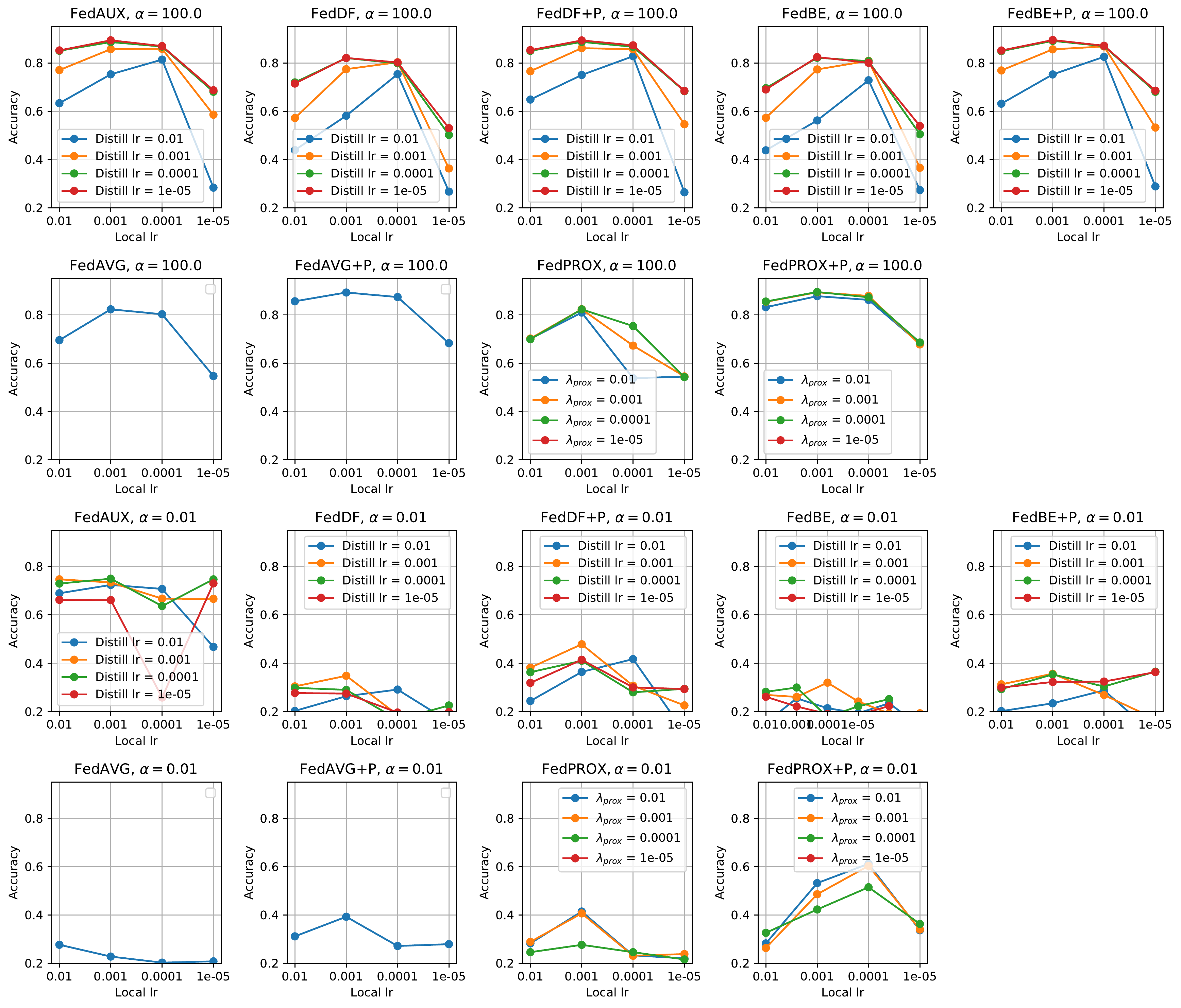}
    \caption{Results of our \textbf{hyperparameter optimization} for ResNet8. 20 Clients are trained for 50 communication rounds, at a participation rate of $C=40\%$. Both local training and distillation is performed for 1 epoch.}
    \label{fig:HPO}
\end{figure*}

\begin{table*}[t!]
    \centering
        \caption{\textbf{Best performing hyperparameter combinations} for each method when training ResNet8 with $n=20$ clients for 50 communication rounds at a participation rate of $C=40\%$. Both local training and distillation is performed for 1 epoch. Methods sorted by top accuracy.}
    \label{tab:HPO}

\begin{tabular}{lrrllr}
\toprule
    Method & Alpha &  Local LR & Distill LR &   $\lambda$ FedProx &  Accuracy \\
\midrule
 FedPROX+P &   100 &    0.001 &          - &            0.0001 &    0.8946 \\
    FedAUX &       &    0.001 &      1e-05 &                 - &    0.8941 \\
   FedDF+P &       &    0.001 &      1e-05 &                 - &    0.8936 \\
  FedAVG+P &       &    0.001 &          - &                 - &    0.8924 \\
     FedBE &       &    0.001 &      1e-05 &                 - &    0.8246 \\
   FedPROX &       &    0.001 &          - &             0.001 &    0.8232 \\
    FedAVG &       &    0.001 &          - &                 - &    0.8228 \\
     FedDF &       &    0.001 &      1e-05 &                 - &    0.8210 \\
     \midrule
    FedAUX &  0.01 &    0.001 &     0.0001 &                 - &    0.7501 \\
 FedPROX+P &       &     0.01 &          - &              0.01 &    0.6122 \\
   FedDF+P &       &    0.001 &      0.001 &                 - &    0.4786 \\
   FedPROX &       &    0.001 &          - &              0.01 &    0.4145 \\
  FedAVG+P &       &    0.001 &          - &                 - &    0.3929 \\
     FedDF &       &    0.001 &      0.001 &                 - &    0.3481 \\
     FedBE &       &    0.001 &      0.001 &                 - &    0.3196 \\
    FedAVG &       &   0.0001 &          - &                 - &    0.2770 \\
\bottomrule
\end{tabular}

\end{table*}

%\begin{figure}
%    \centering
%    \includegraphics[width=0.5\textwidth]{images/hyperparameters_learning_rate.pdf}
%    \caption{Learning Rate.}
%    \label{fig:}
%\end{figure}

\section{Domain-Adaptation-Theoretic Motivation for weighted ensemble distillation}
\label{supp:domain_adaptation}
Domain adaptation theory \suppcite{mansour2008domain, bendavid2010domain, hoffman2018domain}, and in particular with multiple sources, can be used in order to obtain generalization bounds for non-iid FL settings as it has been done in \suppcite{lin2020ensemble} for uniformly averaging of the client hypotheses to obtain a global hypothesis. From multiple-source adaptation theory we know that a distribution-weighted combination of the client hypotheses is robust w.r.t. generalization for any target domain that is a convex combination of the source domains. However, exact information about the local distributions is rarely present in practical applications of FL and if it is, then directly sharing this information with the server in order to get a better global hypothesis is often not feasible in FL settings due to privacy restrictions. Nonetheless, settings with exact or approximate information about the local distributions (e.g. obtained by KDE) show us, what is possible if the server had access to this information and thus leads to benchmarks with a solid theoretic foundation to which we can compare our approach. Consequently, we aim at a weighting of the client’s local hypotheses based on a privacy-preserving local distribution estimation that respects both the theoretical generalization capabilities and the privacy restrictions in FL.

With the help of a toy example in Fig.~\ref{fig:toy_domain_adaptation} we illustrates that the certainty scores $s_i(\cdot), i \in \{1,\ldots,n\}$, obtained via privacy-preserving logistic regression give a good approximation to the distribution-weights suggested by domain adaptation theory \suppcite{mansour2008domain}, i.e. we show that $s_i(x) / \sum_{j} s_j(x) \approx D_i(x) / \sum_j D_j(x)$ for $x \in \mathcal{X}$. 

\begin{figure*}[t!]
    \centering
    \includegraphics[width=\textwidth]{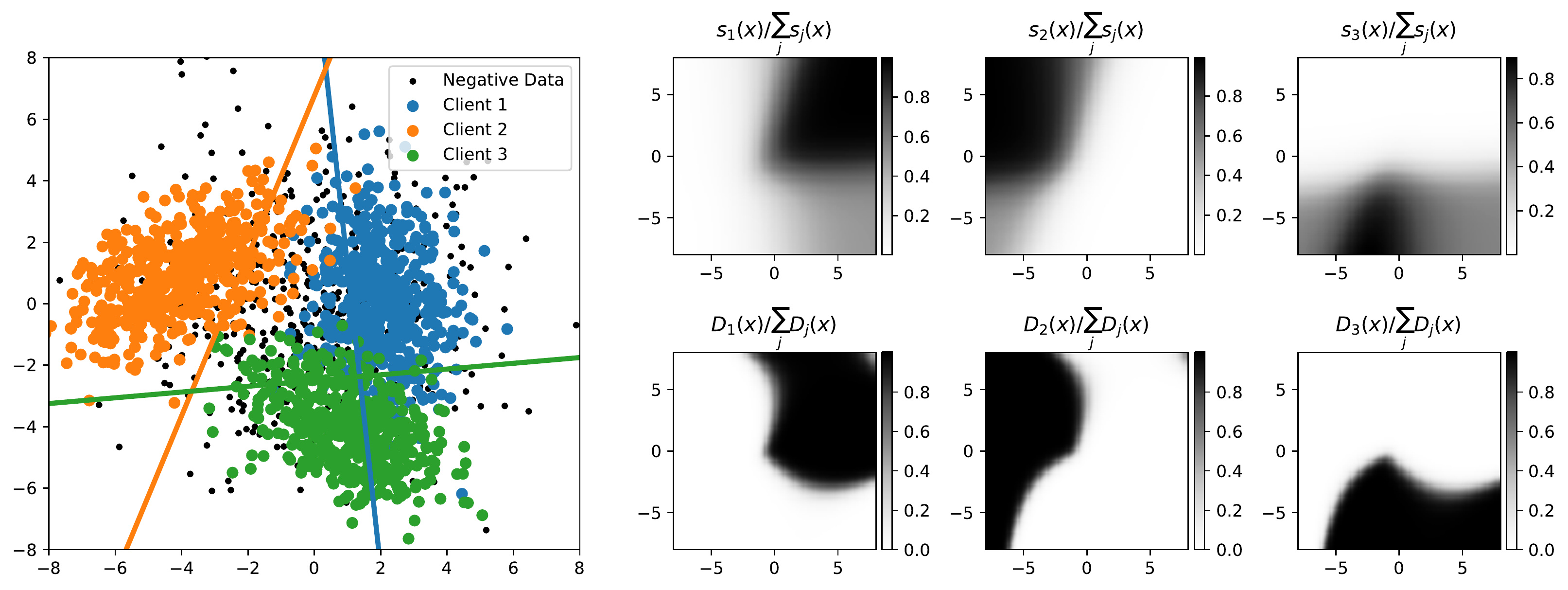}
    \caption{Left: Toy example with 3 clients holding data sampled from multivariate Gaussian distributions $D_1$, $D_2$ and $D_3$. All clients solve optimization problem $J$ by contrasting their local data with the public negative data, to obtain scoring models $s_1$, $s_2$, $s_3$ respectively. As can be seen in the plots to the right, our proposed scoring method approximates the robust weights proposed in \suppcite{mansour2008domain}  as it holds $s_i(x)/\sum_j s_j(x)\approx D_i(x)/\sum_j D_j(x)$ on the support of the data distributions.}
    \label{fig:toy_domain_adaptation}
\end{figure*}

\section{Proof of Theorem \ref{theo:1}}
\label{supp:proof}
\begin{theorem}
If $R(\cdot)$ is differentiable and 1-strongly convex and $l$ is differentiable with $|l'(z)|\leq 1$ $\forall z$, then the $\ell^2$-sensitivity $\Delta_2(\mathcal{M})$ of the mechanism 
\begin{align}
    \mathcal{M} : D_i \mapsto \arg\min_{w} J(w, h_0, D_i, D^-)
\end{align}
is at most $2(\lambda(|D_i|+|D^-|))^{-1}$.
\end{theorem}
\begin{proof}
The proof is an adaptation of the result shown in \suppcite{chaudhuri2011differentially}. 
We have
\begin{align}
\begin{split}
    J(w, h_0, D_i, D^-) = &a\sum_{x\in D_i \cup D^-}l(t_x\langle w, \tilde{h}_0(x)\rangle)+\lambda R(w)
\end{split}
\end{align}
with $t_x=2(\mathbb{1}_{x\in D_i})-1\in[-1,1]$, $a=(|D_i|+|D^-|)^{-1}$ and $\tilde{h}_0(x)=h_0(x)(\max_{x\in D^-\cup D_i} \|h_0(x)\|)^{-1}$.

Let $D_i=\{x_1,..,x_N\}$ and $D_i'=\{x_1,..,x'_N\}$ be two local data sets that differ in only one element. For arbitrary $D^-$ and $h_0$ define
\begin{align}
w^* = \arg\min_{w} J(w, h_0, D_i, D^-),    
\end{align}
\begin{align}
v^* = \arg\min_{w} J(w, h_0, D_i', D^-),   
\end{align}
\begin{align}
    n(w)=J(w, h_0, D_i, D^-)
\end{align}
and
\begin{align}
    m(w) &= J(w, h_0, D_i, D^-)-J(w, h_0, D'_i, D^-)
\end{align}
Since 
\begin{align}
    m(w) =  a (l(t_x\langle w, h_0(x_N)\rangle)-l(t_x\langle w, h_0(x'_N)\rangle))
\end{align}
we have
\begin{align}
    \nabla m(w) = a (t_xl'(t_x\langle w, h_0(x_N)\rangle)h_0(x_N)^T-\\t_xl'(t_x\langle w, h_0(x'_N)\rangle)h_0(x'_N)^T)
\end{align}
which can be bounded in norm 
\begin{align}
    \|\nabla m(w)\| &= a (\|h_0(x_N)-h_0(x_N')\|)\\&\leq a (\|h_0(x_N)\|+\|h_0(x_N')\|)\\&\leq 2a
\end{align}
as $t_x\in[-1,1]$, $|l'(x)|\leq 1$ and 
\begin{align}
\|\tilde{h}_0(x)\|=\|h_0(x)(\max_{x\in D_i\cup D^-}h_0(x))^{-1}\|\leq1.
\end{align}

Furthermore, since $n(w)$ is $\lambda$-strongly convex it follows by Shalev-Schwartz inequality
\begin{align}
    (\nabla n(w^*) - \nabla n(v^*))^T(w^*-v^*)\geq\lambda\|w^*-v^*\|^2 .
\end{align}
Combining this result with Cauchy-Schwartz inequality and $\nabla m(v^*) = \nabla n(v^*) - \nabla n(w^*)$  yields
\begin{align}
    \|w^*-v^*\|\|\nabla m(v^*)\|&\geq (w^*-v^*)^T\nabla m(v^*)\\&=(w^*-v^*)^T(\nabla n(v^*)-\nabla n(w^*))\\&\geq\lambda \|w^*-v^*\|^2 
\end{align}
Thus
\begin{align}
    \|w^*-v^*\|\leq\frac{\|\nabla m(v^*)\|}{\lambda}\leq\frac{2a}{\lambda}
\end{align}
which concludes the proof.
\end{proof}

\section{Empirical Privacy Evaluation}
\label{supp:epirical_privacy}
Our proposed method is provably differentially private and achieves state-of-the-art performance, even at very conservative privacy levels. If not explicitly stated otherwise, all results presented in this study were achieved with $(\varepsilon, \delta)$-differentially private certainty scores at conservative privacy parameters $\delta=10^{-5}$ and $\varepsilon=0.1$. In this section, we additionally evaluate the privacy properties of the certainty scores empirically. Figure \ref{fig:similar_images} shows, for four different clients, the 5 images $x$ from the distillation data set $D_{distill}$, which were assigned the highest certainty score $s_i(x)$ by the client's scoring model $w_i^*$ (left column). Displayed next to the images are their 4 nearest neighbors $x'$ in feature space which maximize the cosine-similarity
\begin{align}
    \text{sim}(x,x') = \frac{\langle h_0(x),  h_0(x')\rangle}{\|h_0(x)\|\|h_0(x')\|}.
\end{align}
In this example the clients hold non-iid subsets of CIFAR-10 ($\alpha=0.01$) and the "Imagenet Dogs" (c.f. Appendix \ref{supp:iamgenet_subsets}) data set is used as auxiliary data. Using weighted ensemble distillation in this setting improves training performance from 48.46\% to 75.59\%. 
As we can see, while certainty scores are able to inform the distillation process and allow \textsc{FedAUX} to outperform baseline methods on heterogeneous data, they reveal only fuzzy, indirect information about the local training data. For instance, client 1, which in this example is mainly holding data from the airplane class, assigns the highest scores to pictures in the auxiliary data set that show dogs in cars or in front of blue skies. From this it could be concluded that a majority of the clients training data contains man-made objects in front of blue backgrounds, but direct exposure of single data points is improbable. 

Note that there exist also many FL scenarios in which the server is assumed to be trustworthy, and only the final trained model which is released to the public needs to be privately sanitized. In these settings, direct inspection of certainty scores by outside adversaries is not possible and thus privacy loss through certainty scores is even less critical. Future work could also explore the use encryption-based techniques for secure weighted aggregation of client predictions.
\begin{figure*}[t!]
    \centering
    \subfigure[Client 1: Images $x$ from the distill data set with the highest scores $s_i(x)$ and their nearest neighbors in feature space in the local data set $D_i$.]{\includegraphics[width=0.48\textwidth]{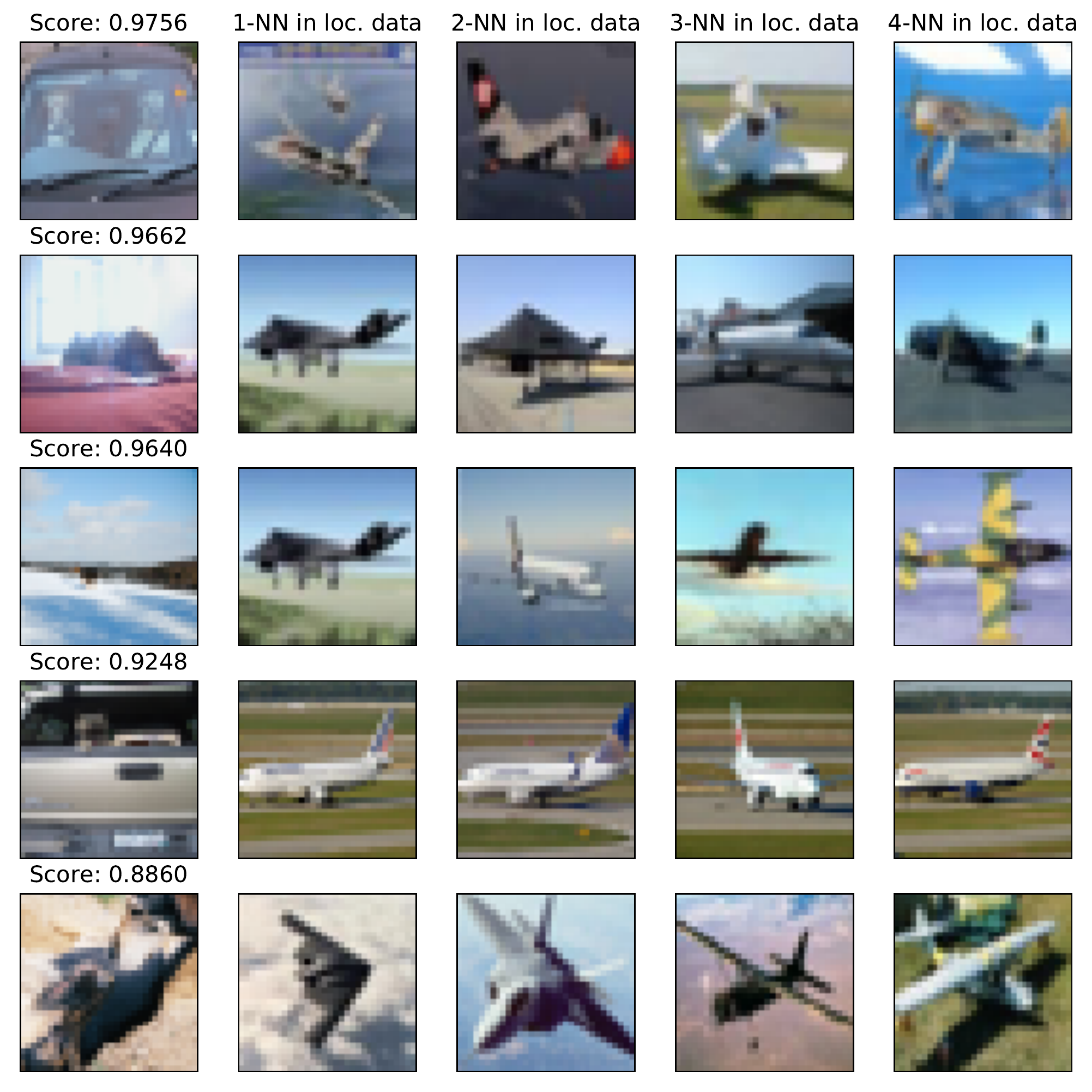}}\hfill
    \subfigure[Client 2: Images $x$ from the distill data set with the highest scores $s_i(x)$ and their nearest neighbors in feature space in the local data set $D_i$.]{\includegraphics[width=0.48\textwidth]{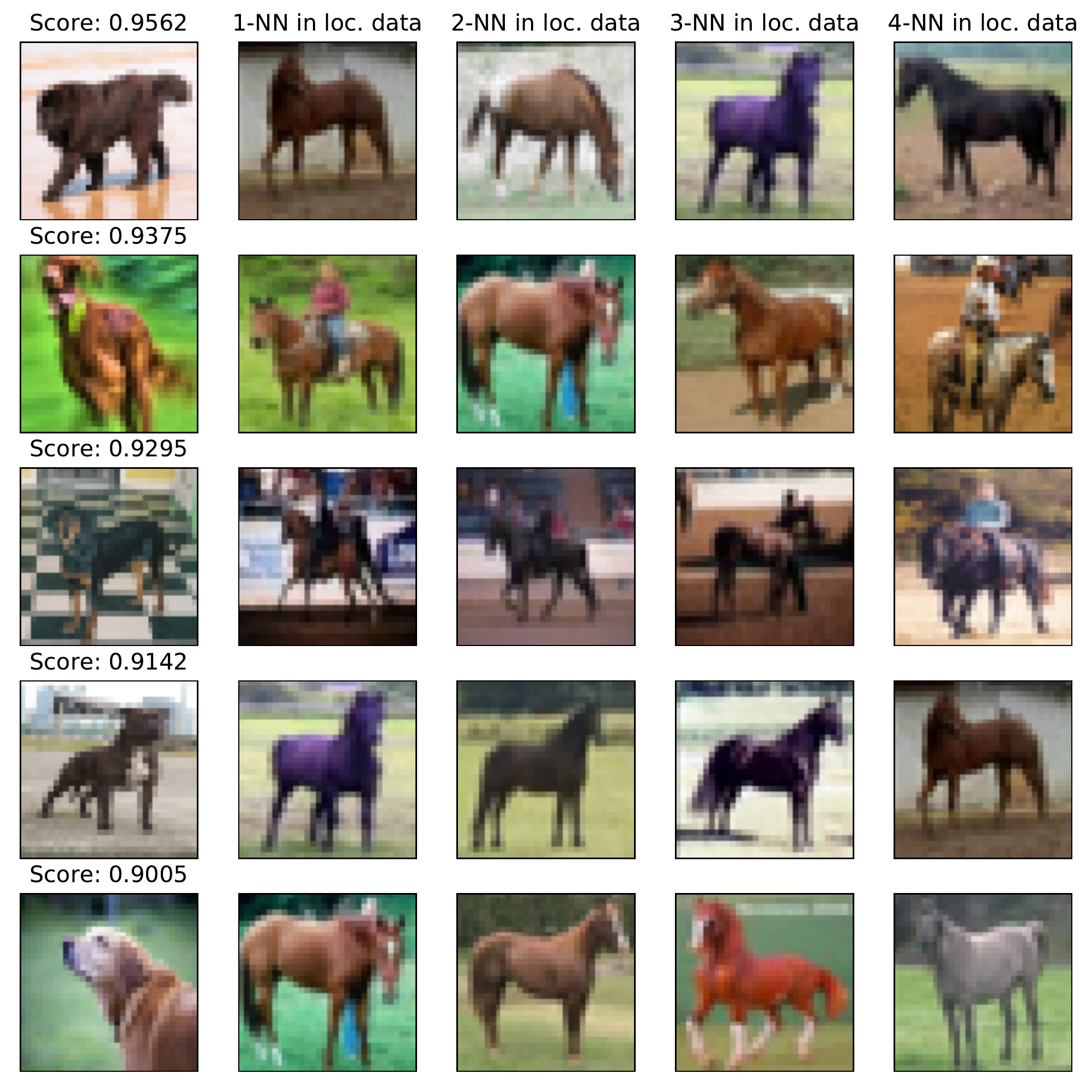}}
    ~
    \subfigure[Client 3: Images $x$ from the distill data set with the highest scores $s_i(x)$ and their nearest neighbors in feature space in the local data set $D_i$.]{\includegraphics[width=0.48\textwidth]{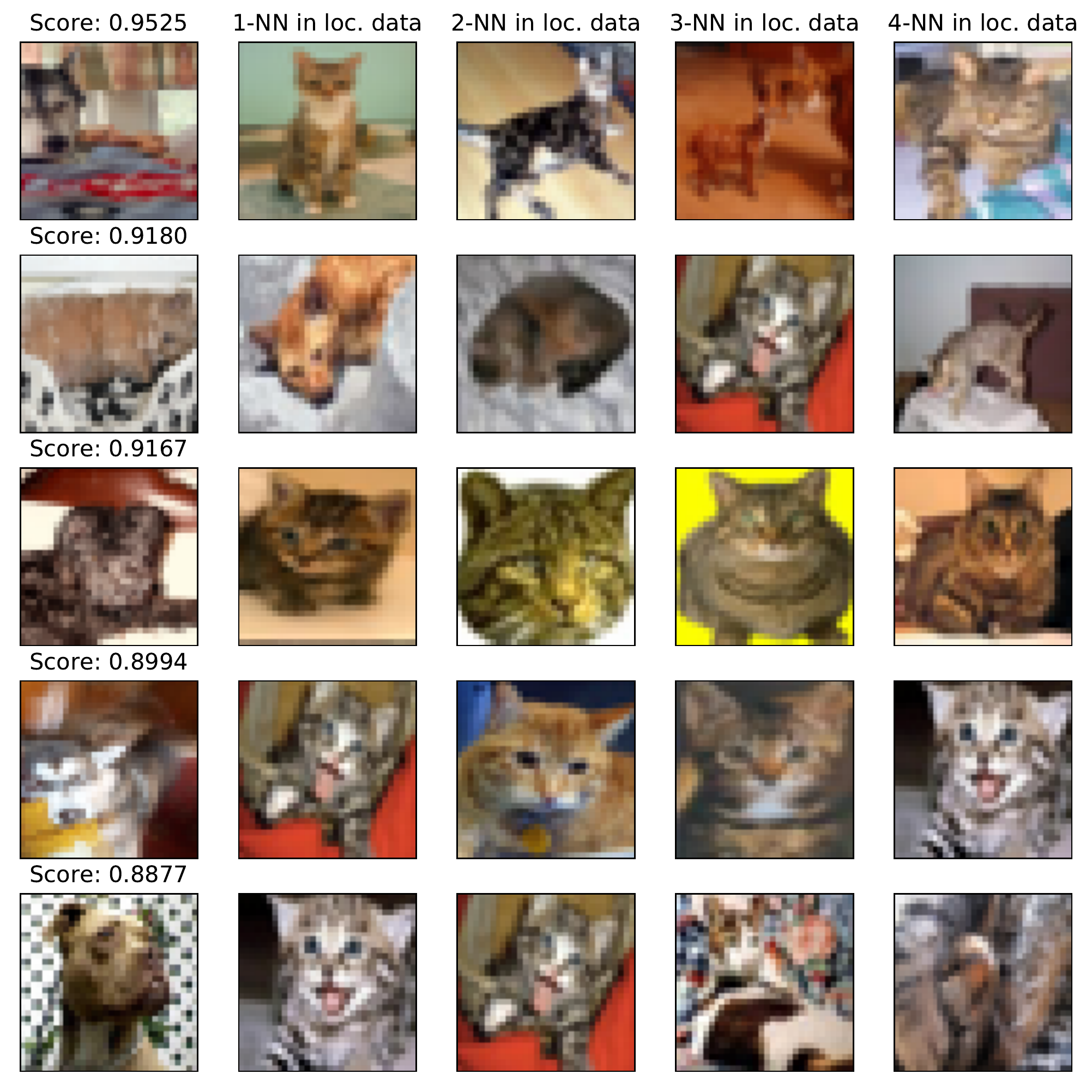}}\hfill
    \subfigure[Client 4: Images $x$ from the distill data set with the highest scores $s_i(x)$ and their nearest neighbors in feature space in the local data set $D_i$.]{\includegraphics[width=0.48\textwidth]{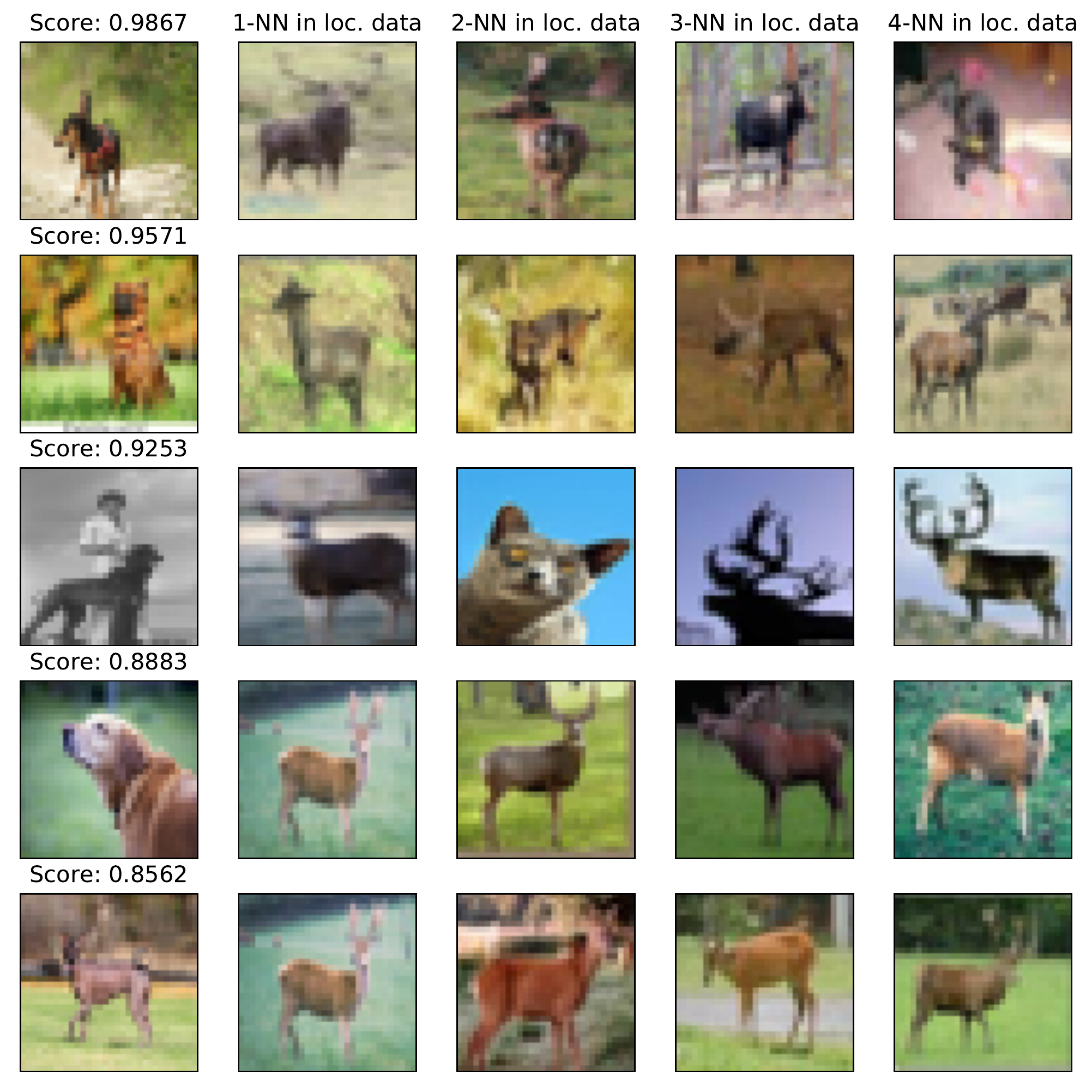}}
    
    \caption{Data points $x$ from the auxiliary data set which were assigned the highest scores $s_i(x)$ and their nearest neighbors in the data of 4 randomly selected clients $D_i$. Clients hold non-iid subsets from the CIFAR-10 data set ($\alpha=0.01$). Auxiliary data used is ImageNet Dogs (cf. Appendix \ref{supp:iamgenet_subsets}). No differential privacy is used.}
    \label{fig:similar_images}
\end{figure*}

\begin{figure*}[t!]
    \centering
    \subfigure[Images from the distill data set with the higher scores and their nearest neighbors in feature space in the local data set of client 1.]{\includegraphics[width=0.48\textwidth]{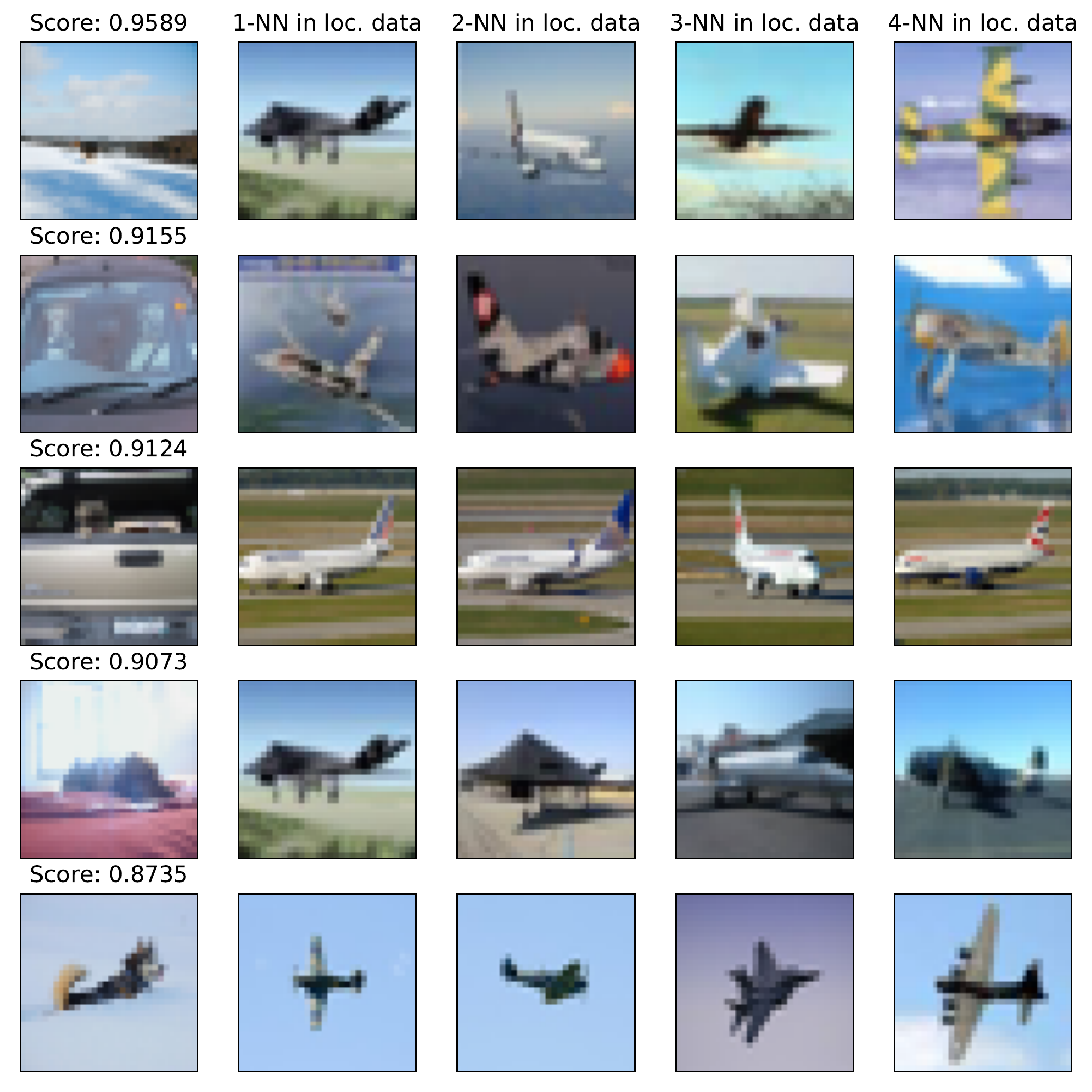}}\hfill
    \subfigure[Images from the distill data set with the higher scores and their nearest neighbors in feature space in the local data set of client 2.]{\includegraphics[width=0.48\textwidth]{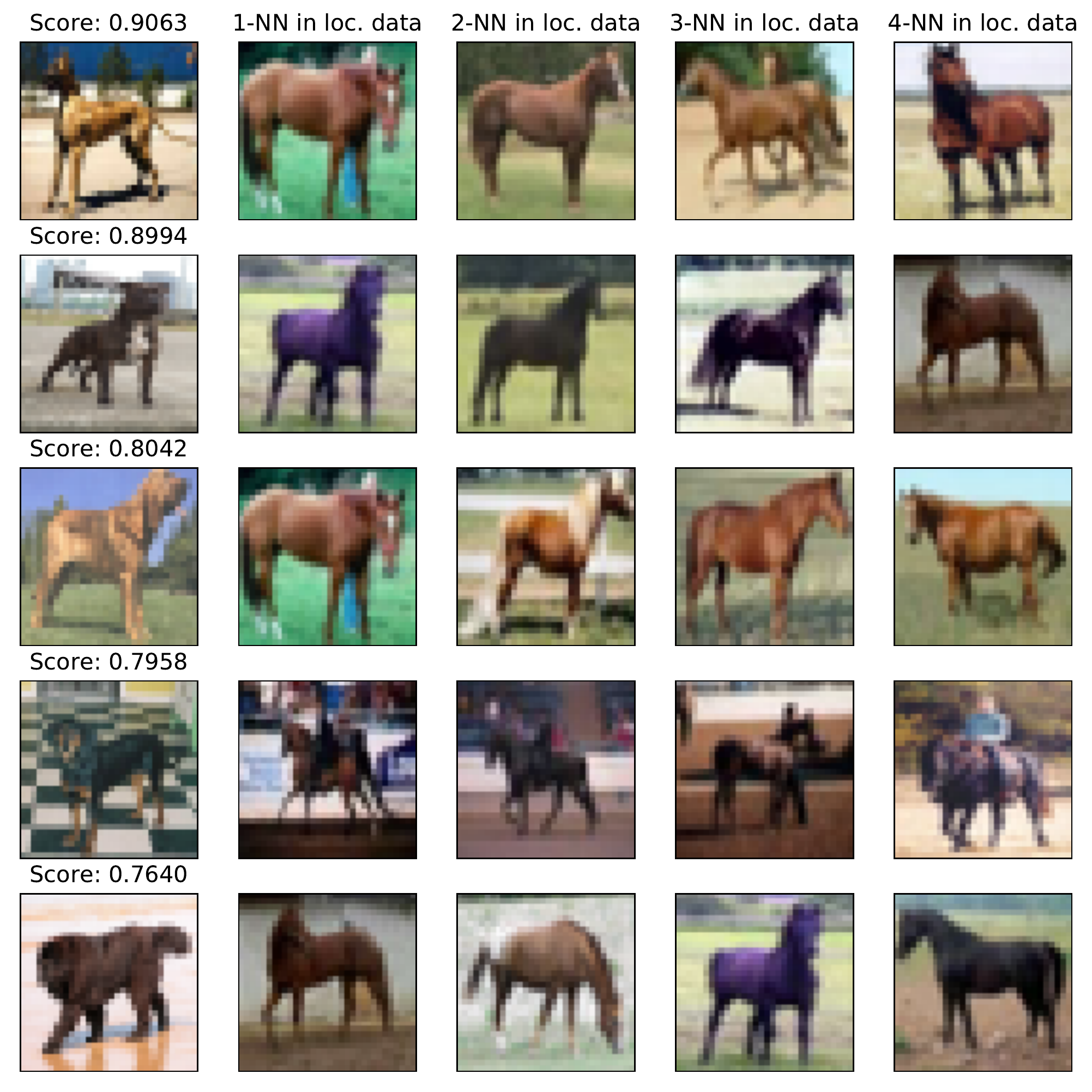}}
    ~
    \subfigure[Images from the distill data set with the higher scores and their nearest neighbors in feature space in the local data set of client 3.]{\includegraphics[width=0.48\textwidth]{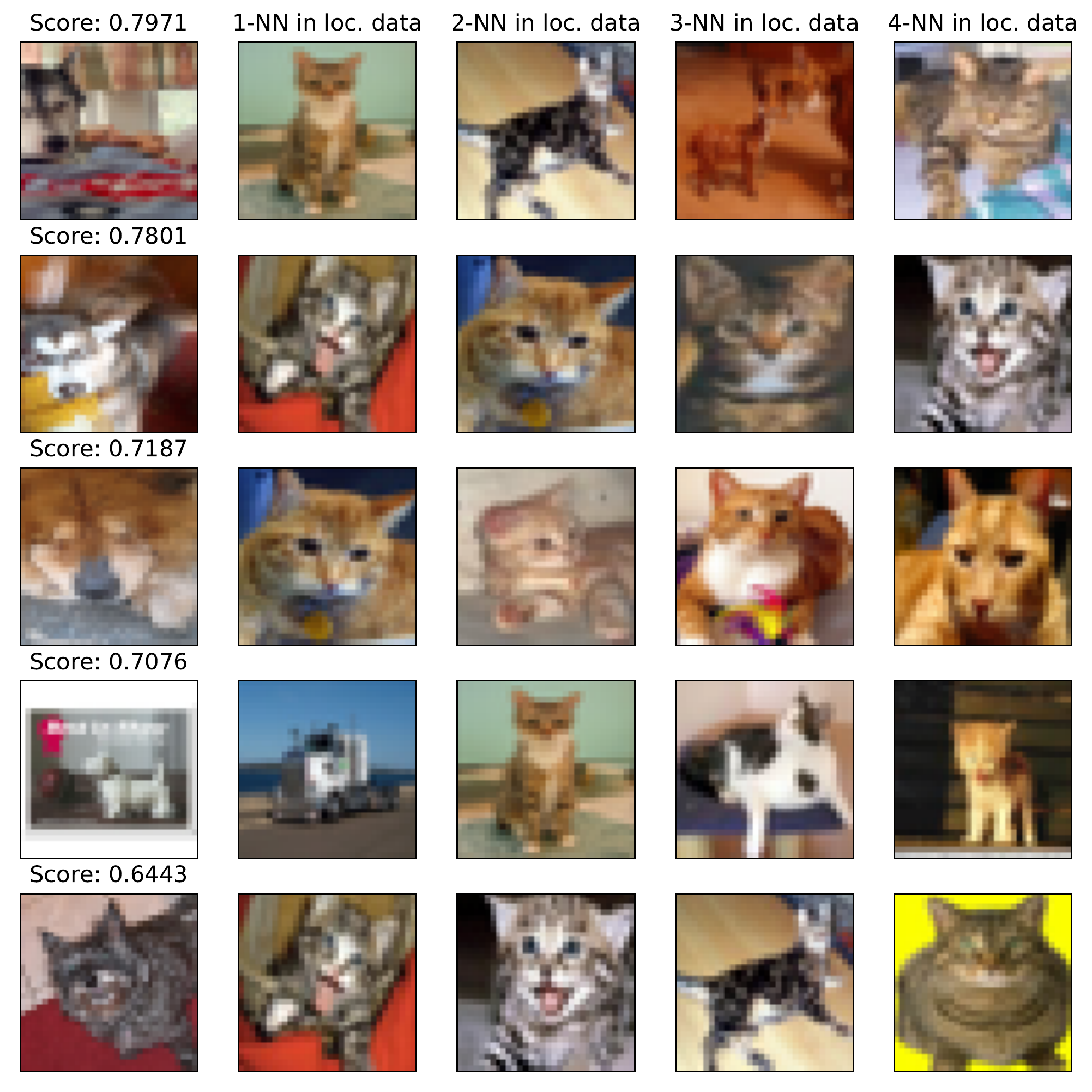}}\hfill
    \subfigure[Images from the distill data set with the higher scores and their nearest neighbors in feature space in the local data set of client 4.]{\includegraphics[width=0.48\textwidth]{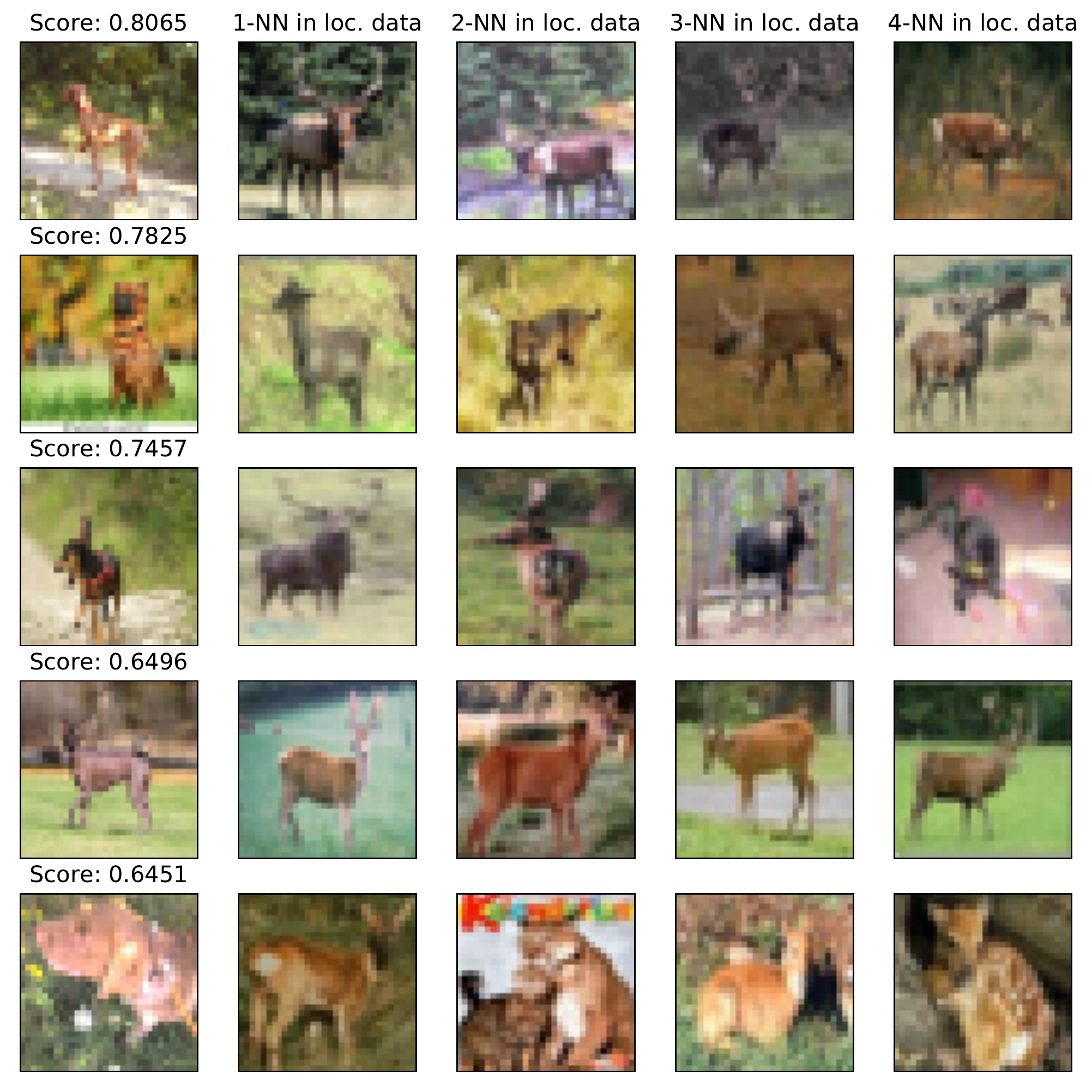}}
    
    \caption{Data points $x$ from the auxiliary data set which were assigned the highest scores $s_i(x)$ and their nearest neighbors in the data of 4 randomly selected clients $D_i$. Clients hold non-iid subsets from the CIFAR-10 data set ($\alpha=0.01$). Auxiliary data used is ImageNet Dogs (cf. Appendix \ref{supp:iamgenet_subsets}).  Scores obtained with differential privacy at $\varepsilon=0.1$, $\delta=10^{-5}$.}
    \label{fig:similar_images_with_dp}
\end{figure*}

\newpage

\supprefs

\end{document}